%% file: main.tex
\documentclass[11pt]{article}
\usepackage[final]{acl}

\usepackage{microtype}
\usepackage{amsmath,amssymb,amsfonts,amsthm}
\newtheorem{lemma}{Lemma}
\newtheorem{theorem}{Theorem}
\newtheorem{assumption}{Assumption}
\usepackage{algorithm}
\usepackage{algorithmic}
\usepackage{booktabs}
\usepackage{graphicx}
\usepackage{subcaption}
\usepackage[table]{xcolor}
\usepackage{siunitx}
\usepackage{enumitem}
\usepackage{wrapfig}
\usepackage{pifont}
\usepackage{xspace}
\usepackage{titletoc}
\usepackage{tcolorbox}
\tcbuselibrary{skins}

\usepackage{fontspec}
\usepackage{polyglossia}
\setmainlanguage{english}
\setotherlanguage{bengali}
\newfontfamily\bengalifont[Script=Bengali]{NotoSerifBengali.ttf}
\newcommand{\beng}[1]{{\bengalifont\mdseries #1}}
\newcommand{\cmark}{{\color{green!60!black}\ding{51}}}
\newcommand{\xmark}{{\color{red!65!black}\ding{55}}}
\newcommand{\herr}[1]{{\color{red!65!black}\textbf{#1}}}
\newcommand{\hcorr}[1]{{\color{green!45!black}\textbf{#1}}}

\definecolor{darkblue}{rgb}{0, 0, 0.5}
\colorlet{rowvanilla}{white}
\colorlet{rowbaseline}{yellow!10}
\colorlet{rowftse}{cyan!8}
\definecolor{tgtBN}{HTML}{2E75B6}
\definecolor{tgtRU}{HTML}{C55A11}
\definecolor{tgtID}{HTML}{538135}
\hypersetup{colorlinks=true, citecolor=darkblue, linkcolor=darkblue, urlcolor=darkblue}

\title{Unveiling Language Routing Isolation in Multilingual MoE Models for Interpretable Subnetwork Adaptation}

\author{Kening Zheng$^1$ \quad Wei-Chieh Huang$^1$ \quad Jiahao Huo$^1$ \quad Zhonghao Li$^3$ \quad Henry Peng Zou$^1$ \\
\textbf{Yibo Yan$^2$ \quad Xin Zou$^2$ \quad Jungang Li$^2$ \quad Junzhuo Li$^2$ \quad Hanrong Zhang$^1$} \\
\textbf{Xuming Hu$^2$ \quad Philip S. Yu$^{1}$\thanks{Corresponding author.}} \\
\\
$^1$University of Illinois Chicago \quad $^2$HKUST (Guangzhou) \quad $^3$University of Maryland}

\newcommand{\rise}{\textsc{RISE}\xspace}

\begin{document}

\maketitle

\begin{abstract}
Mixture-of-Experts (MoE) models exhibit striking performance disparities across languages, yet the internal mechanisms driving these gaps remain poorly understood.
In this work, we conduct a systematic analysis of expert routing patterns in MoE models, revealing a phenomenon we term \textit{Language Routing Isolation}, in which high- and low-resource languages tend to activate largely disjoint expert sets.
Through layer-stratified analysis, we further show that routing patterns first converge and then diverge across model depth.
Building on these findings, we propose \textbf{\rise} (routing isolation-guided subnetwork enhancement), a framework that identifies and adapts language-specific expert subnetworks.
\rise uses a tripartite selection strategy: specificity scores select language-specific experts in shallow and deep layers, while overlap scores select universal experts in middle layers.
By training only the selected subnetwork while freezing the rest of the model, \rise improves low-resource performance while preserving other-language capabilities.
Across 10 languages, \rise achieves target-language F1 gains of up to 10.85\% with minimal cross-lingual degradation.
\end{abstract}

\input{contents/1introduction}

\input{contents/2background}

\input{contents/3routing_analysis}

\input{contents/4method}
\input{contents/5experiments}
\input{contents/6conclusion}

\input{contents/8limitations}

\bibliography{references}

\input{contents/7appendix}

\end{document}

%% file: contents/1introduction.tex
\section{Introduction}
The MoE paradigm has emerged as a powerful architectural approach in modern language modeling, demonstrating remarkable capabilities across a wide range of tasks, including reasoning, code generation, and multilingual understanding~\citep{phi3_2024,deepseekv3_2024,deepseekv2_2024}.
By leveraging sparse expert activation, MoE models enable scaling to tens or hundreds of billions of parameters while maintaining practical inference efficiency, making them particularly attractive for resource-intensive applications.
Despite their empirical success, MoE models exhibit significant performance disparities across languages~\citep{colm2025Crosslingual}.
As shown in Table~\ref{tab:mgsm_two_models}, MoE models show accuracy gaps exceeding 80 percentage points on the MGSM benchmark between high-resource languages (e.g., English) and low-resource languages (e.g., Bengali and Swahili).

Recent studies have explored expert mechanisms in multilingual MoE models.
\citet{bandarkar2026multilingual} analyzed layer-wise activation probability distributions to characterize routing patterns, while \citet{nusSteering} identified top experts based on total activation frequencies to understand expert specialization.
However, these analyses focus either on layer-wise activation distributions or on global activation frequency, without clarifying the distinct functional roles that different model depths play in multilingual processing.

To address this gap, we analyze multilingual routing behavior in MoE models from two complementary perspectives.
At the global level, we uncover the \textit{routing isolation} phenomenon, where high- and low-resource languages rely on nearly orthogonal sets of experts.
Through layer-stratified analysis, we further reveal \textit{layer-wise convergence--divergence}: routing converges in middle layers, which are more language-agnostic, but diverges in shallow and deep layers, which capture language-specific encoding and generation.
This pattern clarifies how MoE models process multilingual inputs across depth and provides a principled basis for multilingual adaptation in practice.

\begin{table*}[t]
\centering
\small
\setlength{\tabcolsep}{5pt}
\begin{tabular}{lrrrrrrrrrrr}
\toprule
\rowcolor{gray!20}
\textbf{Model} & \textbf{BN} & \textbf{DE} & \textbf{EN} & \textbf{ES} & \textbf{FR} & \textbf{JA} & \textbf{RU} & \textbf{SW} & \textbf{TH} & \textbf{ZH} & \textbf{Avg.} \\
\midrule
Phi-3.5-MoE-instruct & 1 & 79 & 88.5 & 81.5 & 73 & 56 & 77 & 1 & 26.5 & 65.5 & 54.9 \\
Qwen3-30B-A3B        & 46 & 88.5 & 96.5 & 91.5 & 82 & 83.5 & 92.5 & 48 & 87.5 & 86 & 80.3 \\
\bottomrule
\end{tabular}
\caption{Accuracy on MGSM across languages. Language codes: BN (Bengali), DE (German), EN (English), ES (Spanish), FR (French), JA (Japanese), RU (Russian), SW (Swahili), TH (Thai), ZH (Chinese).}
\label{tab:mgsm_two_models}
\end{table*}

Building on these findings, we propose \rise\footnote{Our code will be released at \url{https://github.com/init-neok/rise}.} (\textbf{R}outing \textbf{I}solation-guided \textbf{S}ubnetwork \textbf{E}nhancement), a method designed to improve low-resource language performance while preserving the model’s capabilities in high-resource languages.

\rise adopts a hierarchical selection strategy to identify the expert subnetwork that predominantly supports the target language.
Specifically, in shallow and deep layers, \rise uses a specificity score to identify experts that are particularly associated with the target low-resource language.
In middle layers, where routing patterns are more shared across languages, \rise uses an overlap score to select universal experts.
Extensive experiments show that \rise consistently improves target-language performance while preserving performance on other languages and tasks.

In summary, this paper makes the following contributions:

\begin{itemize}

    \item[\ding{182}]\textbf{Empirical finding.} We analyze multilingual routing behavior in MoE models from both global and layer-wise perspectives, uncovering two key phenomena: \textit{routing isolation} and \textit{layer-wise convergence--divergence}. Together, they expose a consistent routing structure that shapes multilingual computation in MoE models and motivates our adaptation strategy in \rise.

    \item[\ding{183}]\textbf{Methodological contribution.} Building on these observations, we propose RISE, a layer-stratified expert selection method for low-resource language adaptation that selects distinctive experts in shallow and deep layers while preserving shared experts in the middle layers.

    \item[\ding{184}]\textbf{Empirical validation.} We conduct comprehensive experiments across 10 languages using multiple datasets, validating both the effectiveness and efficiency of our routing-isolation approach.
\end{itemize}

%% file: contents/2background.tex
\section{Related Work}
\vspace{-2mm}
\paragraph{Expert Specialization and Routing Analysis.} Understanding how experts specialize and how routing mechanisms function has been an active area of research.
Early studies focused on analyzing expert utilization patterns and their correlation with input characteristics~\citep{shazeer2017outrageouslylargeneuralnetworks}.
More recent work has explored the interpretability of expert behaviors in multilingual settings.
\citet{bandarkar2026multilingual} analyze multilingual routing with an entropy-normalized Jensen--Shannon divergence over post-softmax router weights, comparing each non-English sequence with its paired English translation.
\citet{nusSteering} identified top experts based on total activation frequencies and proposed layerwise steering techniques to understand and control expert specialization across languages.

\paragraph{Relation to \citet{bandarkar2026multilingual}.}
Our layer-wise observations are largely confirmatory: that work already reports that routing similarity tracks language similarity, and that early and late layers are more language-specific while middle layers align more strongly across languages.
We differ in the routing signal and in how the resulting structure is used.
Rather than a divergence over post-softmax router weights measured against paired English translations, we compute the Jaccard overlap of the discrete top-$K$ expert support that is \emph{actually executed}, aggregated over a corpus into an all-pairs language matrix instead of an English-referenced comparison.
And rather than intervening on router logits at inference time, we exploit this structure during training, freezing the router and updating only the selected experts.
Our contribution is therefore not the discovery of layer-wise routing structure per se, but the characterization of executed-support isolation across all language pairs and its conversion into a subnetwork selection rule for adaptation.

%% file: contents/3routing_analysis.tex
\section{Routing Analysis}
\label{sec:analysis}
\vspace{-2mm}
To understand the multilingual capability disparities in MoE models, we analyze expert routing patterns from two complementary perspectives: global-level activation statistics and layer-wise routing dynamics.
These two views operate at different granularities and are complementary rather than contradictory: routing isolation is defined over the globally aggregated top-$K$ routed support across all layers, whereas the layer-wise analysis decomposes similarity at each depth relative to a high-resource reference language.

\subsection{\textit{Routing Isolation} at Global Level}

For a given language $l$, let $c_{l,j}$ denote the total activation count of expert $j$ aggregated across all layers and samples.
We select the top-$K$ most frequently activated experts globally:

\begin{equation}
\mathcal{E}_l^{\mathrm{global}}
= \operatorname*{arg\,topK}_{j \in [N_e L]} \; c_{l,j}
\end{equation}

where $N_e$ is the number of experts per layer, $L$ is the total number of MoE layers, and $K$ is the total number of experts considered in the statistics.
To quantify the routing overlap between languages, we compute the Jaccard similarity between their top-$K$ expert sets:

\begin{equation}
J(l_1, l_2) = \frac{|\mathcal{E}_{l_1}^{\text{global}} \cap \mathcal{E}_{l_2}^{\text{global}}|}{|\mathcal{E}_{l_1}^{\text{global}} \cup \mathcal{E}_{l_2}^{\text{global}}|}
\end{equation}

where $J(l_1, l_2) \in [0, 1]$, with 0 indicating no overlap and 1 indicating complete overlap.

We partition languages into high-resource (English, Chinese) and low-resource (Bengali, Swahili, Thai, etc.) groups based on their proportions in the training corpus.
Figure~\ref{fig:global_routing} (a) shows the global expert activation overlap across languages.
We provide more examples of routing isolation, together with reference points that calibrate these overlap magnitudes, in Appendix~\ref{app:additional_routing_analysis}.

\begin{figure*}[t]
\centering
\begin{subfigure}[b]{0.38\textwidth}
    \centering
    \includegraphics[height=4.7cm,keepaspectratio,trim=5 16.5 5 5,clip]{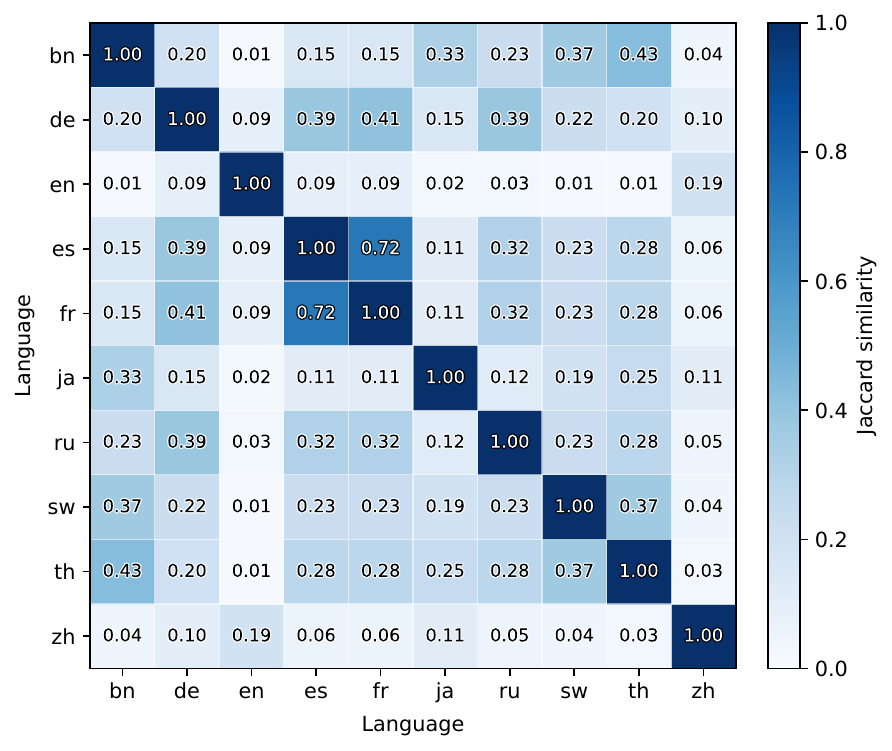}
    \caption{Global-level expert activation overlap across languages in MGSM.}
\end{subfigure}
\hfill
\begin{subfigure}[b]{0.59\textwidth}
    \centering
    \includegraphics[height=4.95cm,keepaspectratio]{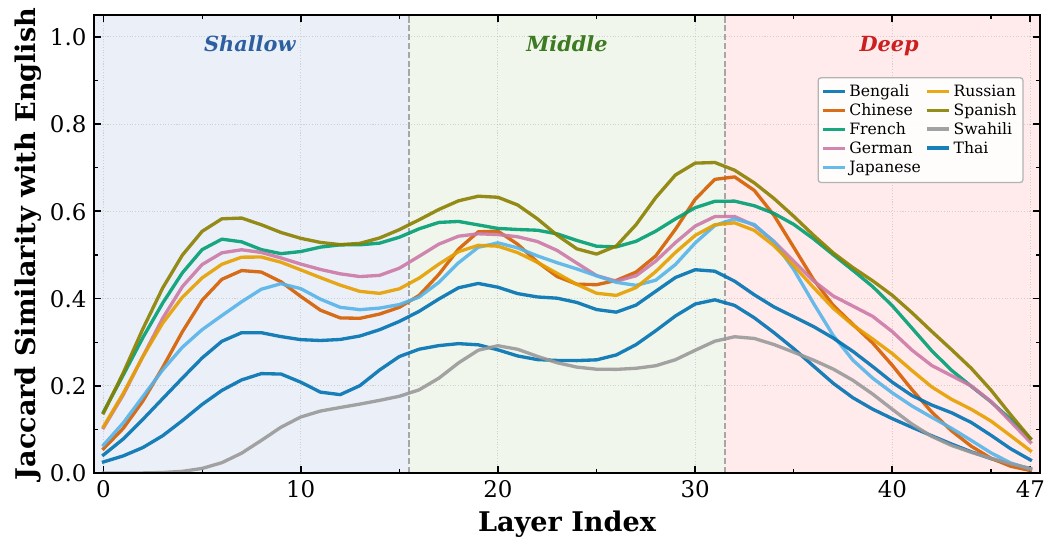}
    \caption{Layer-wise expert routing similarity with English.}
\end{subfigure}
\caption{Comprehensive routing analysis of \texttt{Qwen3-30B-A3B}: global-level (left) expert activation overlap and layer-wise (right) routing similarity with English.}
\label{fig:global_routing}
\end{figure*}

Key observations from the global analysis point to a clear phenomenon between high- and low-resource languages: \ding{182} As high-resource languages, English and Chinese basically don't share experts and rotate among many experts with little overlap in their Top-30 choices, reflecting \textit{flexible, well-trained routing}; \ding{183} Low-resource languages repeatedly \textit{trigger the same small subset}, implying more fixed patterns likely from limited training; \ding{184} overlap across the two resource groups is negligible, showing their expert sets are almost orthogonal, here we call this \textbf{routing isolation}. Besides this, we also found that languages from the same family (e.g., Indo-European) exhibit greater similarity in expert selection patterns, consistent with previous linguistic studies~\citep{tang-etal-2024-language,nusSteering}, thereby grounding our analysis in established theoretical frameworks.

\subsection{\textit{Layer-wise Convergence--Divergence} Pattern}

To understand how expert routing patterns evolve across different model depths, we analyze the layer-wise expert overlap with respect to English as the reference language.
For each layer $i$ within language $l$, we compute the Jaccard similarity with English:

\begin{equation}
J_l(i) \;=\; \frac{\bigl|\mathcal{E}_l^{(i)} \cap \mathcal{E}_{\text{en}}^{(i)}\bigr|}{\bigl|\mathcal{E}_l^{(i)} \cup \mathcal{E}_{\text{en}}^{(i)}\bigr|}
\end{equation}

This yields a per-layer similarity curve $\mathbf{J}_l = \bigl[\,J_l(0),\; J_l(1),\; \ldots,\; J_l(L-1)\,\bigr]$ that captures routing dynamics across layers.
The average similarity within each region $\mathcal{L} \in \{\mathcal{L}_{\text{shallow}}, \mathcal{L}_{\text{middle}}, \mathcal{L}_{\text{deep}}\}$ is:

\begin{equation}
\bar{J}_l^{\,\mathcal{L}} \;=\; \frac{1}{|\mathcal{L}|} \sum_{i \in \mathcal{L}} J_l(i)
\end{equation}

We plot the similarity curve $\mathbf{J}_l$ and compute the average similarity $\bar{J}_l^{\mathcal{L}}$ for each region, as shown in Figure~\ref{fig:global_routing}~(b) and Table~\ref{tab:layerwise_overlap}.

\begin{table}[t]
\centering
\small
\setlength{\tabcolsep}{4pt}
\begin{tabular}{lccc}
\toprule
\textbf{Language} & \textbf{Shallow} & \textbf{Middle} & \textbf{Deep} \\
\midrule
Bengali (BN) & 0.12 & 0.22 & 0.05 \\
Swahili (SW) & 0.04 & 0.19 & 0.04 \\
Thai (TH) & 0.25 & 0.33 & 0.12 \\
Russian (RU) & 0.41 & 0.43 & 0.17 \\
Japanese (JA) & 0.37 & 0.37 & 0.10 \\
\bottomrule
\end{tabular}
\caption{Average layer-wise expert overlap with English. Additional results in Table~\ref{tab:tydiqa_layerwise_overlap}--\ref{tab:phi_mgsm_layerwise_overlap}.}
\label{tab:layerwise_overlap}
\end{table}

As shown in Table~\ref{tab:layerwise_overlap}, routing similarity with English is consistently lower in the shallow and deep layers, while the middle layers exhibit noticeably higher overlap across all languages.
This pattern suggests that the functional stratification observed in dense transformers is also reflected in MoE routing behavior.
Specifically, the shallow and deep layers tend to capture more language-specific processing, whereas the middle layers appear to support more shared cross-lingual representations.
These findings provide a strong empirical basis for our layer-aware expert selection strategy in \rise.

%% file: contents/4method.tex
\section{Method}
\label{sec:method}

Building on the routing isolation identified in Section~\ref{sec:analysis}, which suggests that MoE models contain language-specific expert subnetworks, we present a systematic approach to identify and train these subnetworks for low-resource languages, thereby improving target-language performance while minimally affecting other capabilities. Our method has three stages: (1) collecting routing statistics to understand language-expert affinities, (2) selecting language-specific and shared experts using layer-aware analysis, and (3) training the selected subnetwork while freezing the rest of the model. Figure~\ref{fig:overview} summarizes the full pipeline, and we detail each stage below.

\subsection{Routing Statistics Collection}
\label{sec:routing_collection}

The first stage of \rise involves analyzing the routing behavior of the pre-trained model across multiple languages. This analysis reveals which experts are preferentially activated for different languages, providing the foundation for our selection strategy.

\paragraph{Data Preparation.}
We collect evaluation datasets in $M$ languages $\Lambda = \{\lambda_1, \lambda_2, \ldots, \lambda_M\}$, including the target language $\lambda^*$. These datasets should cover the same or similar tasks to ensure comparable routing patterns.
For each language $\lambda \in \Lambda$, we run inference on the corresponding dataset and record the \emph{discrete routing decisions} for all MoE layers.
Specifically, let $\mathcal{R}^{(l)}_t$ denote the set of experts selected by the router at layer $l$ for token $t$. We define a binary activation indicator $g^{(l)}_{t,i} = \mathbf{1}[i \in \mathcal{R}^{(l)}_t]$ and aggregate it into an activation frequency over all $T_\lambda$ tokens:
\begin{equation}
    a^{(l)}_{\lambda,i} \;=\; \frac{1}{T_\lambda} \sum_{t=1}^{T_\lambda} g^{(l)}_{t,i} \;\in\; [0,1] \label{eq:activation}
\end{equation}
where $a^{(l)}_{\lambda,i}$ represents the empirical probability that expert $i$ is activated when processing language $\lambda$.

This yields a routing profile matrix $\mathbf{A}^{(l)} \in \mathbb{R}^{M \times N}$ for each layer, where each row corresponds to a language and each column to an expert. These matrices capture language--expert affinity patterns that inform our selection strategy.

\subsection{Layer-Aware Expert Selection}
\label{sec:expert_selection}
\begin{figure*}[t]
    \centering
    \includegraphics[width=\linewidth]{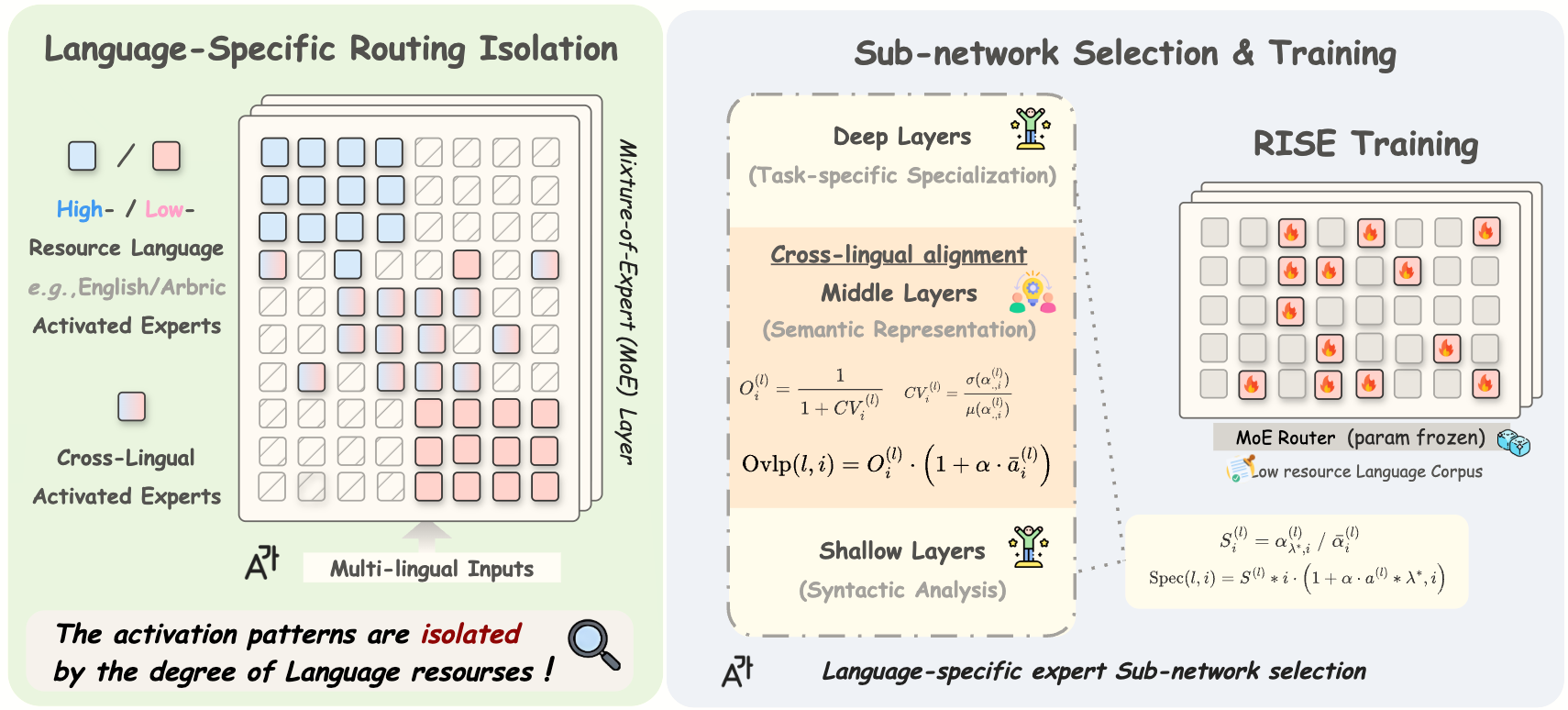}
    \caption{Overview of \rise. (a) We first collect routing statistics across multiple languages. (b) Based on layer-aware analysis, we select language-specific experts in shallow/deep layers and cross-lingual shared experts in middle layers. (c) Only the selected experts are trained.}
    \label{fig:overview}
\end{figure*}

\input{contents/table3_tydiqa}

Based on the analysis in Section~\ref{sec:analysis}, we partition the $L$ MoE layers into three groups: shallow layers $\mathcal{L}_{\text{shallow}} = [0, L_1]$, middle layers $\mathcal{L}_{\text{middle}} = (L_1, L_2]$, and deep layers $\mathcal{L}_{\text{deep}} = (L_2, L-1]$.

\paragraph{Shallow and deep layers.} For shallow and deep layers, we seek experts that are \emph{preferentially activated} by the target language $\lambda^*$. We quantify this preference using the \textbf{specificity score} $S^{(l)}_{\lambda^*,i} = {a^{(l)}_{\lambda^*,i}}/{\bar{a}^{(l)}_i}$, where $\bar{a}^{(l)}_i$ is the average activation probabilities of corresponding expert $i$ in all languages, serving as a baseline for comparison.
The specificity score compares the activation probability of expert $i$ on the target language $\lambda^*$ to its average activation probability across all languages.
When $S^{(l)}_{\lambda^*,i}$ is \emph{greater than} $1$, expert $i$ is preferentially activated by $\lambda^*$, indicating language-specific behavior; when $S^{(l)}_{\lambda^*,i}$ is close to $1$, the expert is activated at a similar rate across languages and thus behaves in a language-agnostic manner; and when $S^{(l)}_{\lambda^*,i}$ is \emph{less than} $1$, the expert is more preferred by other languages.

\paragraph{Middle layers.} For middle layers, we seek experts that capture \emph{language-agnostic} representations shared across all languages. We measure this property using the \textbf{overlap score} based on the coefficient of variation:
\begin{equation}
    O^{(l)}_i = \frac{1}{1 + \mathrm{\widehat{c_{\rm v}}}^{(l)}_i}
    \quad \text{where} \quad
    \mathrm{\widehat{c_{\rm v}}}^{(l)}_i = \frac{\sigma(a^{(l)}_{\cdot,i})}{\mu(a^{(l)}_{\cdot,i})}
    \label{eq:overlap}
\end{equation}

Here, $\widehat{c_{\rm v}}$ represents the coefficient of variation, while $\sigma(\cdot)$ and $\mu(\cdot)$ denote the standard deviation and mean of expert $i$'s activation probabilities across all languages. The coefficient of variation measures relative dispersion:
\begin{itemize}
    \item Low $\widehat{c_{\rm v}}$ (high $O^{(l)}_i$): Uniform activation across languages $\rightarrow$ \emph{shared expert}.
    \item High $\widehat{c_{\rm v}}$ (low $O^{(l)}_i$): Activation varies significantly $\rightarrow$ \emph{language-specific expert}.
\end{itemize}

\paragraph{Subnetwork Selection.} The specificity and overlap scores measure relative preferences but do not account for the absolute activation magnitude. An expert with high specificity but very low overall activation may be less important than one with moderate specificity but high activation. To address this, we define composite scores that incorporate both factors:

\begin{align}
    \mathrm{Spec}(l, i, \lambda^*) &= S^{(l)}_{\lambda^*,i} \cdot \left(1 + \alpha \cdot a^{(l)}_{\lambda^*,i}\right), \label{eq:score_overlap}\\
    \mathrm{Ovlp}(l, i) &= O^{(l)}_i \cdot \left(1 + \alpha \cdot \bar{a}^{(l)}_i\right) \nonumber
\end{align}
where $\alpha > 0$ is a hyperparameter controlling the importance of absolute activation magnitude. In our experiments, we set $\alpha = 10$.

Given a total budget of $K$ experts to train, we allocate them across layer groups using predefined ratios $(\rho_s, \rho_m, \rho_d)$ satisfying $\rho_s + \rho_m + \rho_d = 1$. The complete selection procedure is detailed in Algorithm~\ref{alg:expert_selection}.

With the selected expert set $\mathcal{E} = \{(l_1, i_1), \ldots, (l_K, i_K)\}$, we proceed to train only these experts while keeping all other parameters frozen.

We use standard causal language modeling loss on the target language data $\mathcal{D}_{\lambda^*}$ and our loss function is defined as:
\begin{equation}
    \mathcal{L}(\Theta_{\text{train}}) = -\mathbb{E}_{x \sim \mathcal{D}_{\lambda^*}} \left[ \sum_{t=1}^{|x|} \log P_\Theta(x_t \mid x_{<t}) \right]
\end{equation}
where only $\Theta_{\text{train}}$ receives gradient updates.
We detail the complete procedure in Algorithm~\ref{alg:expert_selection} and visualize the distribution of selected experts in Appendix~\ref{app:ftse_expert_distribution}.

%% file: contents/table3_tydiqa.tex
\begin{table*}[t]
\centering
\small
\setlength{\tabcolsep}{4.1pt}
\begin{tabular}{l
                S[table-format=2.2] S[table-format=2.2] S[table-format=2.2]
                S[table-format=2.2] S[table-format=2.2] S[table-format=2.2]
                S[table-format=2.2] S[table-format=2.2] S[table-format=2.2]
                S[table-format=2.2]}
\toprule
\rowcolor{gray!20}
\textbf{Setting} &
\textbf{AR} & \textbf{BN} & \textbf{EN} & \textbf{FI} & \textbf{ID} & \textbf{KO} & \textbf{RU} & \textbf{SW} & \textbf{TE} & \textbf{Avg.} \\
\midrule
\multicolumn{11}{l}{\textbf{Qwen3-30B-A3B (TyDiQA-GoldP, F1 \%)}} \\
Vanilla         & 49.49 & 51.51 & 24.15 & 22.98 & 26.86 & 48.35 & 27.04 & 27.05 & 58.61 & 37.34 \\
\rowcolor{gray!10}
Random (64, \textcolor{tgtBN}{BN})  & 48.91 & \textbf{\textcolor{tgtBN}{52.17}} & 23.87 & 22.31 & 25.61 & 48.19 & 25.63 & 27.73 & 59.03 & 37.05 \\
Random (128, \textcolor{tgtBN}{BN}) & 50.31 & \textbf{\textcolor{tgtBN}{53.52}} & 24.43 & 22.69 & 26.49 & 48.31 & 26.08 & 27.22 & 59.71 & 37.64 \\
\rowcolor{gray!10}
TopK (128, \textcolor{tgtBN}{BN})   & 50.20 & \textbf{\textcolor{tgtBN}{52.74}} & 24.75 & 23.13 & 27.86 & 47.51 & 26.94 & 26.85 & 59.11 & 37.68 \\
ESFT~\citep{esft}   & 49.51 & \textbf{\textcolor{tgtBN}{51.79}} & 24.08 & 23.26 & 26.89 & 47.33 & 27.09 & 25.85 & 59.70 & 37.28 \\
\midrule
\rowcolor{gray!10}
\rise~(128, \textcolor{tgtBN}{BN})  & 49.15 & \textbf{\textcolor{tgtBN}{54.23}} & 24.04 & 23.39 & 27.61 & 47.77 & 27.22 & 27.40 & 60.39 & 37.91 \\
\rise~(128, \textcolor{tgtRU}{RU})  & 51.06 & 53.62 & 23.88 & 24.15 & 29.55 & 49.08 & \textbf{\textcolor{tgtRU}{29.62}} & 29.30 & 60.66 & 38.99 \\
\rowcolor{gray!10}
\rise~(128, \textcolor{tgtID}{ID})  & 51.59 & 53.97 & 25.65 & 23.58 & \textbf{\textcolor{tgtID}{31.79}} & 47.75 & 29.54 & 30.13 & 59.95 & 39.33 \\
\midrule\midrule
\multicolumn{11}{l}{\textbf{Phi-3.5-MoE-Instruct (TyDiQA-GoldP, F1 \%)}} \\
Vanilla         & 36.75 & 36.04 & 16.94 & 19.19 & 22.09 & 19.75 & 16.19 & 11.55 & 7.68  & 20.69 \\
\rowcolor{gray!10}
Random (16, \textcolor{tgtBN}{BN})  & 38.26 & \textbf{\textcolor{tgtBN}{41.93}} & 18.90 & 19.90 & 23.23 & 26.84 & 18.75 & 11.64 & 9.34  & 23.20 \\
Random (32, \textcolor{tgtBN}{BN})  & 37.53 & \textbf{\textcolor{tgtBN}{43.61}} & 18.89 & 20.10 & 23.28 & 24.56 & 18.45 & 11.44 & 8.93  & 22.98 \\
\rowcolor{gray!10}
TopK (16, \textcolor{tgtBN}{BN})    & 39.63 & \textbf{\textcolor{tgtBN}{49.51}} & 18.32 & 20.78 & 23.59 & 34.58 & 18.91 & 11.56 & 10.78 & 25.30 \\
ESFT~\citep{esft}    & 38.29 & \textbf{\textcolor{tgtBN}{45.44}} & 18.10 & 20.12 & 23.15 & 28.84 & 19.10 & 12.24 & 10.09 & 23.93 \\
\midrule
\rowcolor{gray!10}
\rise~(16, \textcolor{tgtBN}{BN})   & 37.43 & \textbf{\textcolor{tgtBN}{46.89}} & 17.55 & 19.64 & 22.75 & 21.10 & 17.26 & 11.56 & 9.68  & 22.65 \\
\rise~(16, \textcolor{tgtRU}{RU})   & 37.68 & 34.83 & 17.63 & 20.85 & 22.85 & 21.98 & \textbf{\textcolor{tgtRU}{17.81}} & 10.70 & 8.32  & 21.41 \\
\rowcolor{gray!10}
\rise~(16, \textcolor{tgtID}{ID})   & 39.43 & 36.83 & 19.00 & 21.51 & \textbf{\textcolor{tgtID}{22.33}} & 24.50 & 20.78 & 11.03 & 8.49  & 22.66 \\
\bottomrule
\end{tabular}
\vspace{2pt}
\caption{\textbf{TyDiQA-GoldP multilingual QA performance (F1, \%).} Each row corresponds to a training setting; notation \textit{(budget, lang)} denotes the number of experts selected by \rise and the target language, e.g., (128, BN) selects 128 experts routed by Bengali and trains only those experts on Bengali data. The \textcolor{tgtBN}{\textbf{colored bold}} value indicates the target language column for each training setting (\textcolor{tgtBN}{blue}$=$BN, \textcolor{tgtRU}{orange}$=$RU, \textcolor{tgtID}{green}$=$ID); the language code in the first column is colored accordingly.}
\label{tab:tydiqa_main_split}
\end{table*}

%% file: contents/5experiments.tex
\section{Experiments}

In this section, we conduct extensive experiments to answer five research questions. \textbf{RQ1:} Can \rise improve performance on the target low-resource language? \textbf{RQ2:} Does \rise preserve performance on non-target languages and broader tasks? \textbf{RQ3:} Do the experts identified by \rise specialize in the target language? \textbf{RQ4:} Is each layer-group component of the \rise expert subnet indispensable? \textbf{RQ5:} How sensitive is \rise to its key hyperparameters selection?

\subsection{Settings}

\paragraph{Datasets} We evaluate on multilingual benchmarks: \textbf{TyDiQA-GoldP}~\citep{tydiqa}, an extractive QA dataset covering 9 languages (AR, BN, EN, FI, ID, KO, RU, SW, TE); and \textbf{MGSM}~\citep{mgsm}, a math reasoning dataset covering 10 languages (BN, DE, EN, ES, FR, JA, RU, SW, TH, ZH).
To verify that \rise does not degrade general capabilities, we further evaluate on \textbf{TriviaQA}, \textbf{MMLU}, \textbf{HellaSwag}, and \textbf{ARC}.

\paragraph{Comparison.} To demonstrate the effectiveness of \rise, we compare it against random expert selection, top-$K$ most-activated expert selection, and variants that change the expert budget or target language under the \rise framework.
We also include ESFT~\citep{esft} as an expert-selection baseline and LoRA as a representative PEFT baseline in Table~\ref{tab:mgsm_bn_tydiqa_transfer}.
\rise operates at the level of \emph{which} experts to train, while LoRA addresses \emph{how} parameters are updated; the two are orthogonal and can be combined.
Appendix~\ref{app:why_not_lora} provides the full discussion.

\paragraph{Implementation Details.} All experiments were conducted on a single NVIDIA H200 GPU.
For GPU memory cost information, we list the details in Table~\ref{tab:gpu_memory_usage}.
Training was performed for 3 epochs with a per-device batch size of 2 and gradient accumulation of 8 steps (effective batch size of 16), using a learning rate of $2 \times 10^{-5}$ and \texttt{bfloat16} mixed precision.
Further details on the training budget, expert allocation, and layer grouping are provided in Table~\ref{tab:layer_aware_settings} in Appendix~\ref{app:layer_settings}. All hyperparameter settings are held fixed across different target languages and backbone models to ensure a fair comparison.
\subsection{Performance and Cross-lingual Preservation (RQ1 \& RQ2)}

To answer RQ1 and RQ2, we comprehensively compare \rise against other methods across two multilingual MoE models.

Table~\ref{tab:tydiqa_main_split} presents the main results. We report the following observations.
\textbf{Obs.~\ding{182} \rise consistently improves target-language performance.}
As shown in Table~\ref{tab:tydiqa_main_split}, \rise achieves the best Bengali F1 of 54.23\% on \texttt{Qwen3-30B-A3B}, outperforming all competing methods. On \texttt{Phi-3.5-MoE-Instruct}, \rise similarly improves Bengali F1 by 10.85\% over the Vanilla baseline, confirming that routing-based expert selection reliably identifies the parameters most critical for the target language. The gains are consistent across multiple target languages: \rise~(128, RU) and \rise~(128, ID) yield the best per-language F1 on Russian (29.62\%) and Indonesian (31.79\%), respectively, demonstrating that the method generalizes beyond a single low-resource setting.

\begin{table}[t]
\centering
\small
\setlength{\tabcolsep}{3pt}
\begin{tabular}{c c c c c c}
\toprule
\rowcolor{gray!20}
\textbf{Setting} & \textbf{TriviaQA} & \textbf{MMLU} & \textbf{HellaSwag} & \textbf{ARC} & \textbf{Avg.} \\
\midrule
Qwen         & 59.60 & 74.40 & 81.50 & 85.76 & 75.32 \\
\rowcolor{gray!10}
\rise                   & 59.00 & 74.65 & 81.50 & 86.78 & 75.48 \\
$\Delta$               & \textcolor{teal}{$-$0.60} & \textcolor{orange}{$+$0.25} & 0.00 & \textcolor{orange}{$+$1.02} & \textcolor{orange}{$+$0.17} \\
\midrule
\rowcolor{gray!10}
Phi   & 65.50 & 77.50 & 70.70 & 90.85 & 76.14 \\
\rise                   & 65.70 & 76.67 & 70.60 & 91.53 & 76.13 \\
\rowcolor{gray!10}
$\Delta$               & \textcolor{orange}{$+$0.20} & \textcolor{teal}{$-$0.83} & \textcolor{teal}{$-$0.10} & \textcolor{orange}{$+$0.68} & \textcolor{teal}{$-$0.01} \\
\bottomrule
\end{tabular}
\caption{\textbf{The Comparison Performance on General Abilities.} TriviaQA (world knowledge), MMLU (academic reasoning), HellaSwag (commonsense), ARC (science QA).}
\label{tab:general_ability}
\end{table}
\textbf{Obs.~\ding{183} \rise delivers targeted improvements without cross-lingual or cross-task degradation.}
Table~\ref{tab:mgsm_bn_tydiqa_transfer} in Appendix~\ref{app:mgsm_crosstask_transfer} shows that \rise trained exclusively on TyDiQA incurs virtually no loss on the held-out MGSM benchmark: \rise~(128, BN) achieves an average accuracy of 80.7\%, marginally above the Vanilla baseline (80.2\%), confirming that the selected subnetworks capture language-specific rather than task-general computation.
LoRA, by contrast, induces severe cross-lingual interference: on \texttt{Phi-3.5-MoE-Instruct}, Thai accuracy collapses by 16.5 percentage points (26.5\%$\to$10.0\%), and on \texttt{Qwen3-30B-A3B}, Swahili drops by 7.0 points (48.0\%$\to$41.0\%).
General capability is equally well preserved: Table~\ref{tab:general_ability} shows that the benchmark scores change only marginally after \rise training, with per-task shifts bounded within 1.02 points, indicating that selective expert updates leave shared, task-agnostic computation intact.

\subsection{Causal Expert Verification (RQ3)}
\textbf{Obs.~\ding{184} The \rise-selected experts are causally responsible for target-language computation.}
To verify that the identified subnet is genuinely language-specific, we pruned the experts belonging to the selected expert subnetwork in both backbone models and re-evaluated them on TyDiQA-GoldP, as shown in Table~\ref{tab:pruned_experts_tydiqa}.
The target language suffers catastrophic collapse on Bengali, while other languages decline only moderately, mirroring the selectivity observed during training.
This bidirectional evidence, improvement upon activation and collapse upon ablation, confirms that \rise isolates the true language-specific subnetwork rather than a task-general or spuriously correlated subset of parameters.
\begin{table*}[t]
\centering
\small
\setlength{\tabcolsep}{4pt}
\resizebox{\textwidth}{!}{%
\begin{tabular}{llllllllll}
\toprule
\rowcolor{gray!20}
\textbf{Setting} & \textbf{AR} & \textbf{BN} & \textbf{EN} & \textbf{FI} & \textbf{ID} & \textbf{KO} & \textbf{RU} & \textbf{SW} & \textbf{TE} \\
\midrule
\multicolumn{10}{l}{\textbf{Qwen3-30B-A3B (TyDiQA-GoldP, F1 \%)}} \\
Vanilla         & 49.49 & 51.51 & 24.15 & 22.98 & 26.86 & 48.35 & 27.04 & 27.05 & 58.61 \\
\rowcolor{gray!10}
\rise~(128, \textcolor{tgtBN}{BN})  & 49.15 & {\textbf{\textcolor{tgtBN}{54.23}}} & 24.04 & 23.39 & 27.61 & 47.77 & 27.22 & 27.40 & 60.39 \\
Pruned~(128, \textcolor{tgtBN}{BN})
  & 26.83\textsubscript{\scriptsize\textcolor{teal}{$\downarrow$22.32}}
  & {\textbf{\textcolor{tgtBN}{1.00}}}\textsubscript{\scriptsize\textcolor{teal}{$\downarrow$53.23}}
  & 24.63\textsubscript{\scriptsize\textcolor{orange}{$\uparrow$0.59}}
  & 19.37\textsubscript{\scriptsize\textcolor{teal}{$\downarrow$4.02}}
  & 30.99\textsubscript{\scriptsize\textcolor{orange}{$\uparrow$3.38}}
  & 12.47\textsubscript{\scriptsize\textcolor{teal}{$\downarrow$35.30}}
  & 17.87\textsubscript{\scriptsize\textcolor{teal}{$\downarrow$9.35}}
  & 21.09\textsubscript{\scriptsize\textcolor{teal}{$\downarrow$6.31}}
  & 7.78\textsubscript{\scriptsize\textcolor{teal}{$\downarrow$52.61}} \\
\midrule
\multicolumn{10}{l}{\textbf{Phi-3.5-MoE-Instruct (TyDiQA-GoldP, F1 \%)}} \\
Vanilla         & 36.75 & 36.04 & 16.94 & 19.19 & 22.09 & 19.75 & 16.19 & 11.55 & 7.68 \\
\rowcolor{gray!10}
\rise~(16, \textcolor{tgtBN}{BN})   & 37.43 & {\textbf{\textcolor{tgtBN}{46.89}}} & 17.55 & 19.64 & 22.75 & 21.10 & 17.26 & 11.56 & 9.68 \\
Pruned~(16, \textcolor{tgtBN}{BN})
  & 12.03\textsubscript{\scriptsize\textcolor{teal}{$\downarrow$25.40}}
  & {\textbf{\textcolor{tgtBN}{1.85}}}\textsubscript{\scriptsize\textcolor{teal}{$\downarrow$45.04}}
  & 16.53\textsubscript{\scriptsize\textcolor{teal}{$\downarrow$1.02}}
  & 19.40\textsubscript{\scriptsize\textcolor{teal}{$\downarrow$0.24}}
  & 18.79\textsubscript{\scriptsize\textcolor{teal}{$\downarrow$3.96}}
  & 16.63\textsubscript{\scriptsize\textcolor{teal}{$\downarrow$4.47}}
  & 7.85\textsubscript{\scriptsize\textcolor{teal}{$\downarrow$9.41}}
  & 3.46\textsubscript{\scriptsize\textcolor{teal}{$\downarrow$8.10}}
  & 1.02\textsubscript{\scriptsize\textcolor{teal}{$\downarrow$8.66}} \\
\bottomrule
\end{tabular}
}
\caption{\textbf{Causal expert verification on TyDiQA-GoldP (F1, \%).} We compare \rise with its pruned version, where the selected Bengali expert subnetwork is ablated. The \textcolor{tgtBN}{\textbf{colored bold}} value indicates the target-language column.}
\label{tab:pruned_experts_tydiqa}
\end{table*}

\subsection{Indispensability of Layer-Group Components (RQ4)}

\begin{figure*}[h]
\centering
\includegraphics[width=\textwidth]{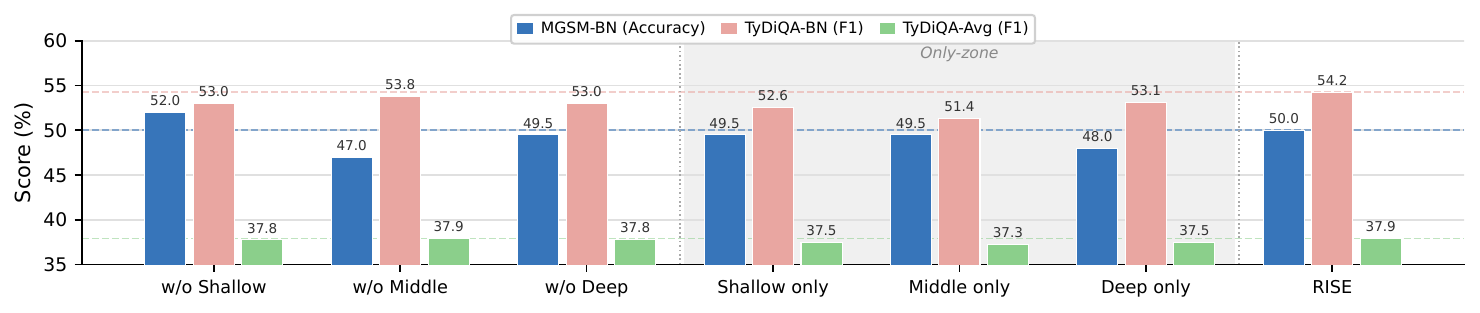}
\caption{Grouped comparison of layer-wise expert subset combinations. \textit{w/o} means removing the corresponding layer group; \textit{Only}: retaining only the corresponding layer group.}
\label{fig:ablation_layer_combo}
\end{figure*}

To answer RQ4, we conduct an ablation study on the contribution of each layer group by training variants of \rise that omit or retain only one group at a time. Results are reported in Figure~\ref{fig:ablation_layer_combo}.
\textbf{Obs.~\ding{185} All three layer groups contribute to \rise, with shallow and deep layers being most critical for target-language adaptation.}
Removing the shallow or deep layer experts causes the biggest TyDiQA-BN performance drops, whereas \textit{w/o Middle} yields the smallest degradation.
The \textit{w/o Middle} result further corroborates the routing analysis in Section~\ref{sec:analysis}: middle layers exhibit language-agnostic, convergent routing and are therefore less critical for language-specific adaptation.
By contrast, experts selected in shallow and deep layers capture language-specific syntactic and semantic patterns that are essential for low-resource performance. These results confirm that the layer-aware design of \rise is non-redundant: each group plays a distinct and necessary role, and no single group alone suffices to match the full method.

\subsection{Sensitivity Analysis (RQ5)}
To answer RQ5, we examine the sensitivity of \rise to key hyperparameters. Figure~\ref{fig:ablation_all} presents the activation scale factor $\alpha$ and the number of selected training experts $K$; the layer-group budget allocation ablation is deferred to Appendix~\ref{app:sensitivity_analysis}.

\begin{figure}[t]
\centering
\begin{subfigure}[t]{0.48\linewidth}
    \centering
    \includegraphics[width=\linewidth]{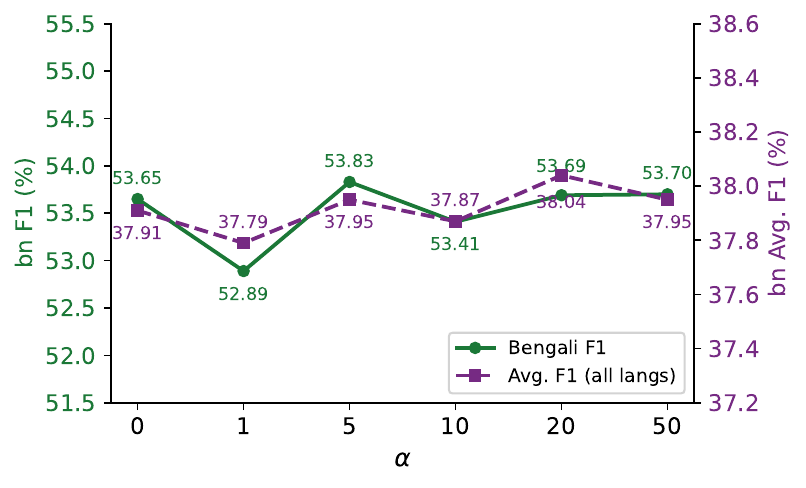}
    \caption{Scale factor $\alpha$.}
    \label{fig:ablation_alpha}
\end{subfigure}
\hfill
\begin{subfigure}[t]{0.48\linewidth}
    \centering
    \includegraphics[width=\linewidth]{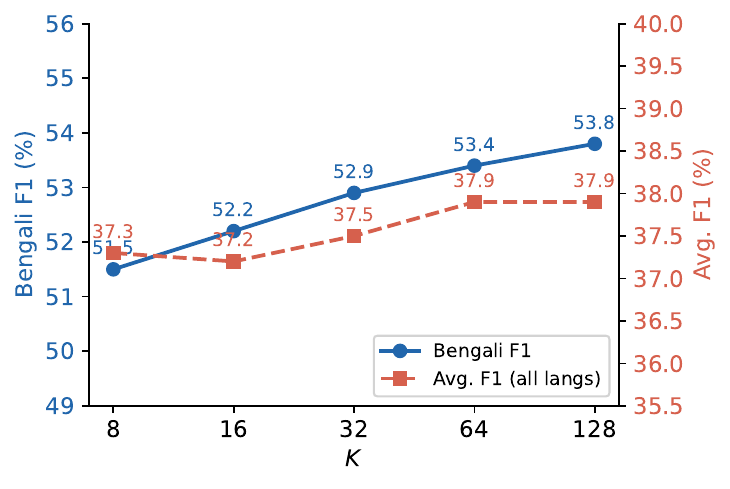}
    \caption{Selected experts $K$.}
    \label{fig:ablation_k}
\end{subfigure}
\caption{Sensitivity analysis for two key hyperparameters in the main text. \textbf{(a)} Activation scale factor $\alpha$. \textbf{(b)} Total number of selected experts $K$.}
\label{fig:ablation_all}
\end{figure}

\textbf{Obs.~\ding{186} \rise is robust to hyperparameter choices within a reasonable operating range.}
As shown in Figure~\ref{fig:ablation_all}(a), performance is stable across a broad range of $\alpha$ values, with a modest optimum that reflects the intended balance between activation frequency and routing weight in the composite selection score. Figure~\ref{fig:ablation_all}(b) demonstrates that performance improves as $K$ increases up to 128, beyond which returns diminish. The budget allocation ablation in Appendix~\ref{app:sensitivity_analysis} shows that concentrating the budget predominantly on shallow and deep layers yields the best results while a uniform allocation underperforms, consistent with the findings in RQ4.

\subsection{Qualitative Case Study}

Appendix~\ref{case_study} presents four qualitative cases comparing the Vanilla baseline and \rise on Bengali MGSM and TyDiQA-GoldP.
Two recurring failure modes are identified: \textbf{mathematical reasoning errors} (e.g., misinterpreting percentage increases, generation truncation) and \textbf{language coherence failures} (e.g., question echoing, incoherent multi-step outputs).
\rise corrects both categories, suggesting that activating the Bengali-specific expert subnetwork improves not only aggregate F1 but also output coherence, consistent with the routing isolation hypothesis.
We provide the full case study in Appendix~\ref{case_study} to illustrate the qualitative improvements enabled by \rise.

\subsection{Composition with LoRA}
\label{sec:rise_lora}

\begin{table}[t]
\centering
\small
\setlength{\tabcolsep}{4pt}
\begin{tabular}{lcccc}
\toprule
& \multicolumn{2}{c}{\textbf{TyDiQA-GoldP}} & \multicolumn{2}{c}{\textbf{MGSM}} \\
\cmidrule(lr){2-3}\cmidrule(lr){4-5}
\textbf{Method} & \textbf{BN} & \textbf{Avg.} & \textbf{BN} & \textbf{Avg.} \\
\midrule
Vanilla & 51.51 & 37.34 & 46.06 & 80.20 \\
\rowcolor{gray!10}
\rise~(128, BN) & 54.23 & \textbf{37.91} & 55.04 & \textbf{80.71} \\
\rise~+ LoRA ($r{=}16$) & \textbf{54.92} & 37.49 & \textbf{56.53} & 80.55 \\
\bottomrule
\end{tabular}
\caption{\textbf{Composing \rise with LoRA on \texttt{Qwen3-30B-A3B}.} LoRA adapters are attached \emph{only} to the \rise-selected expert subnetwork, using the identical selected-expert set and training data as \rise~(128, BN). MGSM is a held-out task, evaluated without task-specific training.}
\label{tab:rise_lora}
\end{table}

\rise and LoRA address orthogonal design questions: \rise determines \emph{which} experts to adapt, a selection problem answered from routing structure, whereas LoRA determines \emph{how} to parameterize the update to whatever is being adapted.
The informative experiment is therefore not to pit them against each other but to compose them, and we report this composition here; Appendix~\ref{app:why_not_lora} details why LoRA is not treated as a competing baseline.
Concretely, we attach LoRA adapters ($r=16$) to the \texttt{gate}/\texttt{up}/\texttt{down\_proj} matrices of the 128 \rise-selected experts (384 adapted linear layers) and freeze everything else.
This run uses the identical selected-expert set and training data as \rise~(128, BN), so the only difference from our main method is full-rank versus low-rank parameterization of the update.

\textbf{The two components compose without disturbing routing-level selectivity.}
As shown in Table~\ref{tab:rise_lora}, applying LoRA inside the selected subnetwork further improves the target language, raising Bengali from 54.23 to 54.92 F1 on TyDiQA-GoldP and from 55.04 to 56.53 on held-out MGSM, at the cost of a marginal 0.42-point decrease in TyDiQA average F1.
Crucially, the selectivity that motivates \rise is retained: held-out MGSM average remains at 80.55, above the 80.20 of the Vanilla baseline, indicating that changing \emph{how} the selected experts are updated does not disturb the gradient isolation that routing-based \emph{selection} provides (Appendix~\ref{gradient_isolation}).
We do not report a parameter-matched LoRA baseline, since LoRA performance does not improve monotonically with rank and matching the trainable-parameter count of 128 fully updated experts would require substantially larger ranks (e.g., $r=128$ or $r=256$), making strict parameter matching of limited diagnostic value.

%% file: contents/6conclusion.tex
\section{Conclusion}
\label{sec:conclusion}

We identify \textbf{language routing isolation} in multilingual MoE models, where high- and low-resource languages naturally follow distinct expert-routing pathways.
Building on this mechanism, we propose \rise, a layer-aware framework that selects and trains target-adaptive expert subnetworks while freezing the rest of the model.
Experiments on TyDiQA and MGSM show that \rise improves low-resource performance while preserving non-target and general capabilities, demonstrating that routing-level interpretability can directly support efficient multilingual adaptation.
Overall, these results show that routing statistics can serve not only as a diagnostic tool for multilingual MoE behavior, but also as a practical guide for selective adaptation under tight parameter budgets.

%% file: contents/8limitations.tex
\section*{Limitations}

Our study covers two open multilingual MoE models, Qwen3-30B-A3B and Phi-3.5-MoE-Instruct, and we do not claim that the layer-wise convergence--divergence pattern is universal: \citet{bandarkar2026multilingual} report that it does not hold uniformly across MoE families (e.g., OLMoE).
\rise does not presuppose a fixed layer profile---it derives the partition from routing statistics profiled on the deployed model---but its advantage depends on expert granularity, since on coarse-grained pools such as Phi-3.5-MoE (top-2 of 16) high-frequency and high-specificity experts largely coincide and simple frequency-based selection may already approximate it.
\rise also adapts only experts the frozen router already assigns to the target language, so it may be insufficient for languages with essentially no routing footprint or tasks whose difficulty lies in shared reasoning.
Finally, Theorems 1 and 2 assume approximately disjoint routing supports and should be read as explaining the observed trend rather than as worst-case guarantees, and our main fine-tuning results are single runs under a shared hyperparameter configuration without significance testing.

%% file: contents/7appendix.tex
\newpage
\appendix
\section*{Appendix}
\startcontents[appendix]
\printcontents[appendix]{}{1}{\setcounter{tocdepth}{2}}
\vspace{8pt}
\section{More Related Work}
\paragraph{Multilingual Language Models}
As large language models have gained increasing prominence, a growing body of research has begun to focus on both the mechanistic interpretability and practical applications of multilingual language models.
\citet{cho2025analyzing} applied sparse autoencoders to dissect multilingual processing in LLMs, revealing how individual features encode language-specific information.
On the data side, \citet{penedo2025fineweb2} introduced FineWeb2, a scalable pre-training data pipeline adapted to every language, enabling the development of high-quality multilingual corpora at scale.
For multilingual encoders, \citet{boizard2025eurobert} proposed EuroBERT, scaling multilingual encoder models across European languages, while \citet{yu2025arcticembed2} presented Arctic-Embed 2.0, achieving strong multilingual retrieval without sacrificing cross-lingual transfer.
Benchmarking efforts have also expanded in scope: \citet{kim2025oneruler} introduced a benchmark for multilingual long-context language models, and \citet{schmidt2025fleursslu} proposed Fleurs-SLU, a massively multilingual benchmark for spoken language understanding.
Beyond the multilingual axis, benchmark design has also targeted specific failure modes of large models: \citet{zheng2025reefknot} construct a benchmark for relation hallucination in multimodal LLMs, together with an evaluation protocol and an analysis of mitigation strategies.
On the evaluation and judgment side, \citet{pombal2025mprometheus} developed M-Prometheus, a suite of open multilingual LLM judges supporting diverse languages.
For safety, \citet{kumar2025polyguard} introduced PolyGuard, a multilingual moderation tool covering 17 languages.
Regarding translation, \citet{ramos2025multilingual} studied multilingual contextualization of LLMs for document-level machine translation.
Collectively, these works highlight the growing importance of multilingual considerations across all stages of language model development, providing important context for our analysis of multilingual routing behavior in MoE models.
\paragraph{Mixture-of-Experts (MoE)} Mixture-of-Experts (MoE) architectures have been widely studied as a means to scale neural networks while maintaining computational efficiency~\citep{shazeer2017outrageouslylargeneuralnetworks,gshard,switchtransformer}.
By activating only a subset of experts for each input token, MoE models can achieve high parameter counts without incurring proportional increases in inference cost.
This design has been particularly influential in the development of large language models (LLMs), where MoE variants have demonstrated strong performance across various natural language processing tasks~\citep{deepseekv2_2024,phi3_2024,qwen3}.
A key component of MoE models is the routing mechanism, which determines how input tokens are assigned to different experts.
Several routing strategies have been proposed, including top-k gating~\citep{shazeer2017outrageouslylargeneuralnetworks}, learned routing networks~\citep{gshard}, and dynamic routing based on input features~\citep{li-etal-2025-dynamic}.
Beyond standard routing, recent work has explored more flexible expert composition strategies: \citet{fleshman2025spectr} proposed SpectR, which dynamically composes experts using spectral routing to improve model expressivity, while \citet{li2025c3po} introduced C3PO, a test-time expert re-mixing approach that optimizes expert pathways at critical layers without retraining.
Complementary efforts have focused on efficiency and compression: \citet{li2025slimmoe} proposed SlimMoE to compress large MoE models via expert slimming and distillation, \citet{song2025blockffn} introduced BlockFFN with chunk-level activation sparsity to accelerate inference on end-side devices, and \citet{bertolissi2025localmixtures} showed that local mixtures of experts can be constructed via model merging to enable essentially free test-time training.

\section{The Algorithm of \rise}

\begin{algorithm}[htbp]
\caption{Language-specific Expert Subnetwork Selection}
\label{alg:expert_selection}
\begin{algorithmic}[1]
\REQUIRE Routing profiles $\{\mathbf{A}^{(l)}\}_{l=0}^{L-1}$, target language $\lambda^*$, budget $K$, layer boundaries $(L_1, L_2)$, allocation ratios $(\rho_s, \rho_m, \rho_d)$
\ENSURE Selected expert set $\mathcal{E}$

\STATE \textbf{Initialize:} $\mathcal{E} \leftarrow \emptyset$
\STATE \textbf{Compute budgets:} $K_s \leftarrow \lfloor K \cdot \rho_s \rfloor$, $K_m \leftarrow \lfloor K \cdot \rho_m \rfloor$, $K_d \leftarrow K - K_s - K_m$

\STATE \textcolor{gray}{\textit{// Phase 1: Select language-specific experts from shallow layers}}
\FOR{each layer $l \in \mathcal{L}_{\text{shallow}}$ and expert $i \in [N]$}
    \STATE Compute $\mathrm{Spec}(l, i, \lambda^*)$ using Eq.~\eqref{eq:score_overlap}
\ENDFOR
\STATE $\mathcal{E} \leftarrow \mathcal{E} \cup \mathrm{TopK}\left(\{(l,i) : l \in \mathcal{L}_{\text{shallow}}\}, K_s\right)$

\STATE \textcolor{gray}{\textit{// Phase 2: Select cross-lingual shared experts from middle layers}}
\FOR{each layer $l \in \mathcal{L}_{\text{middle}}$ and expert $i \in [N]$}
    \STATE Compute $\mathrm{Ovlp}(l, i)$ using Eq.~\eqref{eq:overlap}
\ENDFOR
\STATE $\mathcal{E} \leftarrow \mathcal{E} \cup \mathrm{TopK}\left(\{(l,i) : l \in \mathcal{L}_{\text{middle}}, (l,i) \notin \mathcal{E}\}, K_m\right)$

\STATE \textcolor{gray}{\textit{// Phase 3: Select language-specific experts from deep layers}}
\FOR{each layer $l \in \mathcal{L}_{\text{deep}}$ and expert $i \in [N]$}
    \STATE Compute $\mathrm{Spec}(l, i, \lambda^*)$ using Eq.~\eqref{eq:score_overlap}
\ENDFOR
\STATE $\mathcal{E} \leftarrow \mathcal{E} \cup \mathrm{TopK}\left(\{(l,i) : l \in \mathcal{L}_{\text{deep}}, (l,i) \notin \mathcal{E}\}, K_d\right)$

\RETURN $\mathcal{E}$
\end{algorithmic}
\end{algorithm}

\section{Gradient Isolation}
\label{gradient_isolation}

In this appendix, we formalize why routing isolation leads to gradient isolation, and why
training only the target-language expert subnetwork preserves the model's general
capabilities. The key point is that sparse MoE routing does not merely reduce computation:
it also induces a structured partition of the parameter space. Once the routing supports of
different languages are separated, gradient updates become correspondingly localized.

\paragraph{Setup.}
Consider an MoE transformer with $L$ MoE layers and $N_e$ experts per layer. For an input
sequence $x=(x_1,\dots,x_T)$, let $h_t^{(l)}$ denote the hidden state of token $t$ at layer $l$.
At MoE layer $l$, the output is
\begin{equation}
h_t^{(l+1)}
=
U^{(l)}\!\left(h_t^{(l)}\right)
+
\sum_{i=1}^{N_e}
g_{t,i}^{(l)}(x)\,
F_i^{(l)}\!\left(h_t^{(l)};\theta_{l,i}\right),
\label{eq:moe_forward_appendix}
\end{equation}
where:
(i) $U^{(l)}$ collects all frozen shared transformations at layer $l$ (e.g., residual pathway,
attention block, layer norm, and any non-expert shared modules),
(ii) $F_i^{(l)}(\cdot;\theta_{l,i})$ is expert $i$ at layer $l$ with parameters $\theta_{l,i}$,
and (iii) $g_{t,i}^{(l)}(x)\in\{0,1\}$ is the discrete routing indicator, with
$\sum_{i=1}^{N_e} g_{t,i}^{(l)}(x)=k$ under top-$k$ routing.

Let $S=\{S_l\}_{l=1}^L$ denote the selected expert subnetwork, where
$S_l\subseteq \{1,\dots,N_e\}$ is the set of trainable experts in layer $l$.
During RISE training, only $\theta_S=\{\theta_{l,i}: i\in S_l\}$ are updated; all other
parameters are frozen.

For a language $\lambda$, define its routing support at layer $l$ as
\begin{equation}
\begin{aligned}
\mathcal{R}_{\lambda}^{(l)}
&=
\left\{
i \in \{1,\dots,N_e\} :
\right.\\
&\quad\left.
\Pr_{x\sim \mathcal{D}_{\lambda},\,t}\!\big[g_{t,i}^{(l)}(x)=1\big] > 0
\right\}.
\end{aligned}
\label{eq:routing_support}
\end{equation}
Intuitively, $\mathcal{R}_{\lambda}^{(l)}$ contains the experts that language $\lambda$ actually uses
at layer $l$.

We also define the routing overlap mass between language $\lambda$ and the selected
subnetwork $S$:
\begin{equation}
\begin{aligned}
\Omega_{\lambda}(S)
&=
\sum_{l=1}^{L}\sum_{i\in S_l} p_{\lambda,i}^{(l)}, \\
p_{\lambda,i}^{(l)}
&=
\mathbb{E}_{x\sim \mathcal{D}_{\lambda},\,t}\big[g_{t,i}^{(l)}(x)\big].
\end{aligned}
\label{eq:overlap_mass}
\end{equation}
When routing isolation is strong, $\Omega_{\lambda}(S)$ is small for non-target languages.

\subsection{Gradient Isolation}

We first state the exact gradient form for expert parameters.

\begin{lemma}[Exact gradient isolation]
For any expert $i$ at layer $l$, the gradient of the training loss
$\mathcal{L}_{\lambda^*}$ on target-language data $\mathcal{D}_{\lambda^*}$ satisfies
\begin{equation}
\nabla_{\theta_{l,i}} \mathcal{L}_{\lambda^*}
=
\mathbb{E}_{x\sim \mathcal{D}_{\lambda^*}}
\left[
\sum_{t=1}^{T}
g_{t,i}^{(l)}(x)\,
\nabla_{\theta_{l,i}} \ell_t(x)
\right],
\label{eq:grad_isolation_expectation}
\end{equation}
where $\ell_t(x)$ is the token-level loss contribution.
Hence, if expert $i$ is never activated by target-language tokens, i.e.,
$g_{t,i}^{(l)}(x)=0$ for all $(x,t)$ from $\mathcal{D}_{\lambda^*}$, then
\begin{equation}
\nabla_{\theta_{l,i}} \mathcal{L}_{\lambda^*}=0.
\end{equation}
\end{lemma}

\begin{proof}
From the MoE forward definition in Eq.~\eqref{eq:moe_forward_appendix}, the parameter
$\theta_{l,i}$ appears only in the term
$g_{t,i}^{(l)}(x)F_i^{(l)}(h_t^{(l)};\theta_{l,i})$.
Since $g_{t,i}^{(l)}(x)\in\{0,1\}$ is a multiplicative routing mask, the chain rule gives
\begin{equation}
\nabla_{\theta_{l,i}} \ell_t(x)
=
g_{t,i}^{(l)}(x)\,
\frac{\partial \ell_t(x)}{\partial F_i^{(l)}}\,
\frac{\partial F_i^{(l)}(h_t^{(l)};\theta_{l,i})}{\partial \theta_{l,i}}.
\end{equation}
Summing over tokens and taking expectation over $x\sim \mathcal{D}_{\lambda^*}$ yields
Eq.~\eqref{eq:grad_isolation_expectation}. If $g_{t,i}^{(l)}(x)=0$ for all target-language
examples, every term vanishes identically, so the gradient is exactly zero.
\end{proof}

This lemma shows that sparse routing induces exact gradient sparsity: an expert receives
gradient only from the tokens that actually traverse it.

\subsection{Exact Preservation Under Disjoint Routing}

We now formalize the strongest case: if the selected target-language subnetwork is disjoint
from the routing support of another language, then training on the target language leaves
that other language \emph{exactly unchanged}.

\begin{theorem}[Exact invariance under disjoint routing]
\label{thm:exact_invariance}
Let $\lambda^*$ be the target language and let $S=\{S_l\}_{l=1}^{L}$ be the selected trainable
expert subnetwork. Assume:

\begin{enumerate}
    \item All shared parameters and all unselected experts are frozen.
    \item For a non-target language $\lambda$, the selected experts are disjoint from its
    routing supports at every layer:
    \begin{equation}
    S_l \cap \mathcal{R}_{\lambda}^{(l)} = \varnothing,
    \qquad \forall l=1,\dots,L.
    \label{eq:strict_disjointness}
    \end{equation}
\end{enumerate}

Then, after any number of gradient-based training steps on $\mathcal{D}_{\lambda^*}$ that
update only $\theta_S$, the model's forward computation on any input
$x\sim \mathcal{D}_{\lambda}$ remains exactly unchanged:
\begin{equation}
f_{\theta'}(x)=f_{\theta}(x),
\qquad \forall x\in \mathrm{supp}(\mathcal{D}_{\lambda}),
\label{eq:exact_same_logits}
\end{equation}
where $\theta$ denotes the initial parameters and $\theta'$ the parameters after training.
Consequently, the predictive distribution, token losses, and task outputs on language
$\lambda$ are all unchanged.
\end{theorem}

\begin{proof}
We prove the claim by induction over training steps and model layers.

\paragraph{Step 1: one update cannot alter the forward path of language $\lambda$.}
Consider one optimization step that updates only $\theta_S$:
\begin{equation}
\theta'_{l,i}
=
\begin{cases}
\theta_{l,i} - \eta \nabla_{\theta_{l,i}}\mathcal{L}_{\lambda^*}, & i\in S_l,\\
\theta_{l,i}, & i\notin S_l.
\end{cases}
\end{equation}
Take any input $x\sim \mathcal{D}_{\lambda}$.
We show by induction on layer depth that all hidden states remain unchanged.

For the first layer, the input embeddings are unchanged because they are frozen, so
$h_t^{(0)\prime}=h_t^{(0)}$.

Assume $h_t^{(l)\prime}=h_t^{(l)}$ for all tokens $t$ at layer $l$.
Since the router and all shared modules are frozen, the router receives the same input hidden
states before and after the update; therefore the routing decisions are identical:
\begin{equation}
g_{t,i}^{(l)\prime}(x)=g_{t,i}^{(l)}(x), \qquad \forall i,t.
\end{equation}
By the disjointness assumption in Eq.~\eqref{eq:strict_disjointness}, if
$g_{t,i}^{(l)}(x)=1$, then necessarily $i\notin S_l$, because no selected expert belongs to
the routing support of language $\lambda$. Hence every expert actually activated by
language $\lambda$ at layer $l$ remains frozen:
\begin{equation}
\begin{aligned}
\theta'_{l,i}&=\theta_{l,i} 
&\text{for all } i \text{ such that } g_{t,i}^{(l)}(x)=1.
\end{aligned}
\end{equation}
Therefore, every active expert output is unchanged, and so is the shared branch:
\begin{align}
h_t^{(l+1)\prime}
&=
U^{(l)}(h_t^{(l)\prime})
+
\sum_{i=1}^{N_e}
g_{t,i}^{(l)\prime}(x)\,
F_i^{(l)}(h_t^{(l)\prime};\theta'_{l,i})\\
&=
U^{(l)}(h_t^{(l)})
+
\sum_{i=1}^{N_e}
g_{t,i}^{(l)}(x)\,
F_i^{(l)}(h_t^{(l)};\theta_{l,i})\\
&=
h_t^{(l+1)}.
\end{align}
Thus the induction closes, giving $h_t^{(l)\prime}=h_t^{(l)}$ for all layers $l$ and tokens $t$,
and therefore $f_{\theta'}(x)=f_\theta(x)$.

\paragraph{Step 2: extension to multiple updates.}
The argument above applies after each gradient step individually. Since at every step only
selected experts are modified, and these experts are never traversed by language $\lambda$,
the forward computation for $\lambda$ remains unchanged after every step. By induction over
optimization steps, Eq.~\eqref{eq:exact_same_logits} holds after arbitrary-length training.
\end{proof}

Theorem~\ref{thm:exact_invariance} gives the cleanest formal statement of the intuition:
if a language never uses the updated experts, then it is mathematically impossible for its
forward pass to change.

\subsection{Approximate Preservation Under Near-Orthogonal Routing}

In practice, routing isolation is strong but not perfectly disjoint. We therefore derive a
stability bound showing that cross-lingual interference is controlled by routing overlap.

\begin{assumption}[Expert smoothness]
For each expert $F_i^{(l)}$, there exists a constant $L_{l,i}>0$ such that for all relevant
hidden states $h$ and parameter perturbations $\Delta\theta_{l,i}$,
\begin{equation}
\begin{aligned}
\left\|
F_i^{(l)}(h;\theta_{l,i}+\Delta\theta_{l,i})
-
\right.\\
&\quad\left.
F_i^{(l)}(h;\theta_{l,i})
\right\|
\\
&\le
L_{l,i}\,\|\Delta\theta_{l,i}\|.
\end{aligned}
\label{eq:expert_lipschitz}
\end{equation}
\end{assumption}

\begin{assumption}[Task-head smoothness]
The mapping from final hidden states to logits/loss is $C$-Lipschitz.
\end{assumption}

\begin{theorem}[Cross-lingual perturbation bound]
\label{thm:approx_bound}
Let $\Delta\theta$ be any update supported only on the selected subnetwork $S$, i.e.,
$\Delta\theta_{l,i}=0$ for $i\notin S_l$.
Assume the routing pattern of a non-target language $\lambda$ is locally stable under this
small update. Then the expected output perturbation on $\mathcal{D}_{\lambda}$ is bounded by
\begin{multline}
\mathbb{E}_{x\sim\mathcal{D}_{\lambda}}
\big[
\|f_{\theta+\Delta\theta}(x)-f_{\theta}(x)\|
\big] \\
\le
C
\sum_{l=1}^{L}\sum_{i\in S_l}
p_{\lambda,i}^{(l)}\,L_{l,i}\,\|\Delta\theta_{l,i}\|.
\label{eq:output_perturbation_bound}
\end{multline}
Consequently, the expected loss change satisfies
\begin{multline}
\mathbb{E}_{x\sim\mathcal{D}_{\lambda}}
\big[
|\ell(x;\theta+\Delta\theta)-\ell(x;\theta)|
\big] \\
\le
C'
\sum_{l=1}^{L}\sum_{i\in S_l}
p_{\lambda,i}^{(l)}\,L_{l,i}\,\|\Delta\theta_{l,i}\|
\label{eq:loss_perturbation_bound}
\end{multline}
for some constant $C'>0$.
\end{theorem}

\begin{proof}
Fix $x\sim\mathcal{D}_{\lambda}$ and condition on its routing pattern.
Because only experts in $S$ are updated, the change in the MoE contribution at layer $l$
can only come from activated experts in $S_l$. Using Eq.~\eqref{eq:moe_forward_appendix},
define the per-expert perturbation
\[
\begin{aligned}
\Delta F_{t,i}^{(l)}
&\mathrel{:=}
F_i^{(l)}(h_t^{(l)};\theta_{l,i}+\Delta\theta_{l,i}) \\
&\quad -
F_i^{(l)}(h_t^{(l)};\theta_{l,i}) .
\end{aligned}
\]
The triangle inequality then gives
\begin{equation}
\begin{aligned}
\|h_t^{(l+1)\prime}-h_t^{(l+1)}\|
&\le
\sum_{i\in S_l}
g_{t,i}^{(l)}(x)\,
\bigl\|\Delta F_{t,i}^{(l)}\bigr\|.
\end{aligned}
\end{equation}
Applying the smoothness assumption in Eq.~\eqref{eq:expert_lipschitz} gives
\begin{equation}
\begin{aligned}
\|h_t^{(l+1)\prime}-h_t^{(l+1)}\|
&\le
\sum_{i\in S_l}
g_{t,i}^{(l)}(x)\,L_{l,i}\,\|\Delta\theta_{l,i}\|.
\end{aligned}
\end{equation}
Summing across layers and propagating to the output using the $C$-Lipschitz property of the
remaining frozen network yields
\begin{multline}
\|f_{\theta+\Delta\theta}(x)-f_{\theta}(x)\|
\\
\le
C
\sum_{l=1}^{L}\sum_{i\in S_l}
g_{t,i}^{(l)}(x)\,L_{l,i}\,\|\Delta\theta_{l,i}\|.
\end{multline}
Taking expectation over $x\sim \mathcal{D}_{\lambda}$ and using
$\mathbb{E}[g_{t,i}^{(l)}(x)] = p_{\lambda,i}^{(l)}$ gives
Eq.~\eqref{eq:output_perturbation_bound}. The loss bound in
Eq.~\eqref{eq:loss_perturbation_bound} follows immediately from the Lipschitz continuity of
the loss with respect to logits.
\end{proof}

Theorem~\ref{thm:approx_bound} shows that cross-lingual interference is not arbitrary: it is
proportional to the extent to which a non-target language routes through the updated
experts. Therefore, when routing isolation makes $p_{\lambda,i}^{(l)}$ very small for
$i\in S_l$, the induced perturbation on language $\lambda$ is correspondingly small.

\subsection{Implication for RISE}

The above results justify the design principle of RISE. The method first identifies a
target-language expert subnetwork using routing statistics, and then updates only that
subnetwork while freezing all remaining parameters. If the selected experts align with the
target language's routing support and have little overlap with other languages, then:

\begin{enumerate}
    \item gradients from target-language training are concentrated on the selected experts;
    \item non-target languages do not backpropagate through those experts, or do so only
    with very small probability;
    \item thus, non-target behavior is either exactly preserved (in the disjoint case) or
    perturbed only by a small amount bounded by routing overlap.
\end{enumerate}

In this sense, routing isolation induces \emph{functional decoupling} between language-specific
expert subnetworks, while gradient isolation turns this decoupling into \emph{optimization
locality}. This explains why selective expert training can improve the target language
without materially degrading the model's general multilingual capabilities.

\section{Cross-Task Transfer: MGSM Results After TyDiQA Training}
\label{app:mgsm_crosstask_transfer}
\begin{table*}[htbp]
\centering
\small
\setlength{\tabcolsep}{4.2pt}
\begin{tabular}{l
                S[table-format=2.1] S[table-format=2.1] S[table-format=2.1]
                S[table-format=2.1] S[table-format=2.1] S[table-format=2.1]
                S[table-format=2.1] S[table-format=2.1] S[table-format=2.1]
                S[table-format=2.1] S[table-format=2.1]}
\toprule
\rowcolor{gray!20}
\textbf{Setting} &
\textbf{BN} & \textbf{DE} & \textbf{EN} & \textbf{ES} & \textbf{FR} & \textbf{JA} & \textbf{RU} & \textbf{SW} & \textbf{TH} & \textbf{ZH} & \textbf{Avg.} \\
\midrule
\multicolumn{12}{l}{\textbf{Qwen3-30B-A3B (MGSM, Accuracy \%)}} \\
Vanilla         & 46.0 & 88.5 & 96.5 & 91.5 & 82.0 & 83.5 & 92.5 & 48.0 & 87.5 & 86.0 & 80.2 \\
\rowcolor{gray!10}
Random (64,\textcolor{tgtBN}{BN})  & \textbf{\textcolor{tgtBN}{46.0}} & 87.5 & 96.0 & 90.5 & 81.5 & 84.5 & 91.0 & 46.5 & 87.0 & 89.0 & 80.0 \\
Random (128,\textcolor{tgtBN}{BN}) & \textbf{\textcolor{tgtBN}{49.0}} & 87.5 & 96.0 & 90.0 & 81.5 & 83.0 & 91.0 & 46.0 & 87.0 & 87.0 & 79.8 \\
\rowcolor{gray!10}
LoRA~\citep{hu2022lora}           & \textbf{\textcolor{tgtBN}{65.0}} & 88.0 & 97.0 & 89.5 & 79.0 & 82.0 & 90.0 & 41.0 & 87.5 & 88.5 & 80.8 \\
ESFT ~\citep{esft}  & \textbf{\textcolor{tgtBN}{48.0}} & 86.5 & 97.0 & 91.0 & 82.0 & 84.0 & 92.5 & 47.5 & 87.0 & 86.0 & 80.2 \\
\midrule
\rowcolor{gray!10}
\rise~(128, \textcolor{tgtBN}{BN})  & \textbf{\textcolor{tgtBN}{55.0}} & 87.5 & 96.5 & 89.5 & 82.5 & 84.5 & 91.0 & 46.5 & 86.0 & 87.5 & 80.7 \\
\rise~(128, \textcolor{tgtRU}{RU})  & 49.0 & 88.5 & 96.0 & 91.5 & 81.0 & 85.0 & \textbf{\textcolor{tgtRU}{91.0}} & 43.0 & 87.5 & 89.0 & 80.2 \\
\rowcolor{gray!10}
\rise~(128, \textcolor{tgtID}{ID})  & 46.5 & 88.5 & 97.5 & 90.5 & 82.0 & 84.5 & 92.0 & 47.0 & 88.0 & 87.0 & 80.4 \\
\midrule\midrule
\multicolumn{12}{l}{\textbf{Phi-3.5-MoE-Instruct (vLLM + chat template, MGSM, Accuracy \%)}} \\
Vanilla         & 1.0  & 79.0 & 88.5 & 81.5 & 73.0 & 56.0 & 77.0 & 1.0  & 26.5 & 65.5 & 54.9 \\
\rowcolor{gray!10}
Random (16,\textcolor{tgtBN}{BN})  & \textbf{\textcolor{tgtBN}{3.0}}  & 79.5 & 89.5 & 83.0 & 79.0 & 60.0 & 82.0 & 1.5  & 25.0 & 65.5 & 56.8 \\
Random (32,\textcolor{tgtBN}{BN})  & \textbf{\textcolor{tgtBN}{3.0}}  & 78.5 & 88.0 & 83.5 & 79.5 & 57.0 & 78.5 & 1.5  & 22.5 & 64.5 & 55.7 \\
\rowcolor{gray!10}
TopK (16,\textcolor{tgtBN}{BN})    & \textbf{\textcolor{tgtBN}{1.5}}  & 78.5 & 89.0 & 82.0 & 74.0 & 59.5 & 79.0 & 2.5  & 24.5 & 67.5 & 55.8 \\
LoRA~\citep{hu2022lora} & \textbf{\textcolor{tgtBN}{0.0}} & 82.0 & 91.5 & 80.5 & 73.5 & 60.5 & 78.5 & 0.5 & 10.0 & 64.0 & 54.1 \\
\rowcolor{gray!10}
ESFT ~\citep{esft}   & \textbf{\textcolor{tgtBN}{3.5}}  & 81.0 & 88.5 & 82.5 & 78.0 & 56.5 & 80.5 & 1.5  & 22.5 & 66.0 & 56.1 \\
\midrule
\rise~(16, \textcolor{tgtBN}{BN})   & \textbf{\textcolor{tgtBN}{3.5}}  & 76.5 & 88.5 & 84.0 & 75.5 & 62.5 & 81.5 & 0.5  & 22.0 & 64.0 & 55.9 \\
\rowcolor{gray!10}
\rise~(16, \textcolor{tgtRU}{RU})   & 2.5  & 79.5 & 91.0 & 83.5 & 75.5 & 58.0 & \textbf{\textcolor{tgtRU}{79.0}} & 0.0  & 19.5 & 66.0 & 55.5 \\
\rise~(16, \textcolor{tgtID}{ID})   & 0.5  & 80.5 & 88.5 & 83.5 & 76.5 & 58.5 & 78.0 & 1.5  & 19.5 & 63.5 & 55.1 \\
\bottomrule
\end{tabular}
\vspace{2pt}
\caption{\textbf{Cross-task transfer: MGSM accuracy (\%) for models trained on TyDiQA.} All models are trained exclusively on TyDiQA and evaluated zero-shot on MGSM to measure cross-task interference. The table is split by backbone: \texttt{Qwen3-30B-A3B} and \texttt{Phi-3.5-MoE-Instruct}. The \textcolor{tgtBN}{\textbf{colored bold}} value indicates the target language column (\textcolor{tgtBN}{blue}$=$BN, \textcolor{tgtRU}{orange}$=$RU, \textcolor{tgtID}{green}$=$ID); the language code in the first column is colored accordingly. }
\label{tab:mgsm_bn_tydiqa_transfer}
\end{table*}

To further validate that \rise does not impair the model's capabilities on tasks beyond the training domain, we evaluate all TyDiQA-trained models on the MGSM benchmark without any MGSM-specific training.
This cross-task evaluation directly tests whether the language-specific subnetworks identified by \rise are truly isolated: training on extractive QA (TyDiQA) should not degrade multilingual mathematical reasoning (MGSM) if the selected experts are genuinely language-specific rather than task-general.
As shown in Table~\ref{tab:mgsm_bn_tydiqa_transfer}, \rise consistently preserves MGSM performance across all non-target languages, confirming that our subnetwork selection respects the functional boundaries between language-specific and task-general computation.

\section{Additional Sensitivity Analysis}
\label{app:sensitivity_analysis}

Figure~\ref{fig:ablation_budget_ratio} complements the main-text sensitivity analysis by varying the layer-group budget allocation ratio across shallow, middle, and deep experts.
The best setting concentrates the budget on shallow and deep layers, while the uniform allocation underperforms, matching the layer-wise ablation in Figure~\ref{fig:ablation_layer_combo}.

\begin{figure}[H]
\centering
\includegraphics[width=\linewidth,trim=5 15 5 5,clip]{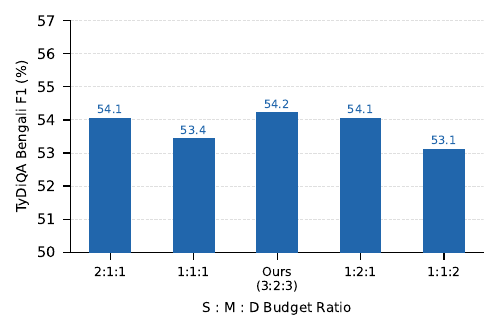}
\caption{Ablation on expert budget allocation ratio across shallow, middle, and deep layers.}
\label{fig:ablation_budget_ratio}
\end{figure}

\section{Additional Routing Analysis}
\label{app:additional_routing_analysis}

This section consolidates the complete layer-wise expert overlap statistics across both backbone models and both evaluation benchmarks, extending the analysis summarized in Table~\ref{tab:layerwise_overlap} of the main text.

\subsection{Calibrating Overlap Magnitudes}
\label{app:overlap_calibration}

\begin{table}[t]
\centering
\small
\setlength{\tabcolsep}{3pt}
\begin{tabular}{@{}lc@{}}
\toprule
\rowcolor{gray!20}
\textbf{Reference point} & \textbf{Jaccard} \\
\midrule
Split-half self-overlap, same language & 0.95 \\
\rowcolor{gray!10}
Same language, MGSM vs.\ TyDiQA (SW / BN) & 0.77 / 0.67 \\
Closest language pair in set (ES--FR) & 0.741 \\
\rowcolor{gray!10}
High-resource pair (EN--ZH) & 0.243 \\
Low-resource pairs, mean & 0.297 \\
\rowcolor{gray!10}
High- vs.\ low-resource pairs, mean & 0.181 \\
Random top-$K$ floor & 0.011 \\
\bottomrule
\end{tabular}
\caption{\textbf{Reference points for interpreting global routing overlap}, computed on \texttt{Qwen3-30B-A3B} over MGSM at $K=128$.}
\label{tab:overlap_calibration}
\end{table}

A Jaccard value is only interpretable relative to what the metric can attain in a given setting, so we report a set of reference points in Table~\ref{tab:overlap_calibration}, all computed on \texttt{Qwen3-30B-A3B} over MGSM at $K=128$.
The upper end is given by split-half self-overlap, where a single language is profiled on two disjoint halves of its own data, reaching $0.95$; profiling the same language on two different benchmarks reaches $0.77$ (Swahili) and $0.67$ (Bengali), indicating that domain shift alone accounts for a substantial part of the gap from unity.
The closest language pair in our set, ES--FR, reaches $0.741$, and the two high-resource languages (EN--ZH) reach $0.243$ under the same protocol.
Against these reference points, low-resource languages overlap with one another at a mean of $0.297$ and with high-resource languages at $0.181$, while a random top-$K$ baseline gives only $0.011$.
Shared routing among low-resource languages is therefore far from accidental---roughly $27\times$ the random floor---yet well short of the near-identical support exhibited by typologically close pairs.
\begin{table}[t]
\centering
\small
\setlength{\tabcolsep}{4pt}
\begin{tabular}{lcccc}
\toprule
\rowcolor{gray!20}
\textbf{Language} & \textbf{Shal.} & \textbf{Mid.} & \textbf{Deep} & \textbf{Avg.} \\
\midrule
Telugu (TE)     & 0.16 & 0.29 & 0.06 & 0.17 \\
\rowcolor{gray!10}
Swahili (SW)    & 0.14 & 0.30 & 0.07 & 0.17 \\
Bengali (BN)    & 0.19 & 0.31 & 0.10 & 0.20 \\
\rowcolor{gray!10}
Korean (KO)     & 0.27 & 0.39 & 0.12 & 0.26 \\
Finnish (FI)    & 0.29 & 0.41 & 0.14 & 0.28 \\
\rowcolor{gray!10}
Arabic (AR)     & 0.31 & 0.42 & 0.21 & 0.31 \\
Indonesian (ID) & 0.32 & 0.46 & 0.19 & 0.32 \\
\rowcolor{gray!10}
Russian (RU)    & 0.40 & 0.42 & 0.21 & 0.34 \\
\midrule
\textbf{Mean}   & \textbf{0.26} & \textbf{0.38} & \textbf{0.14} & \textbf{0.26} \\
\bottomrule
\end{tabular}
\caption{Layer-wise expert overlap with English on TyDiQA (Qwen3-30B-A3B).}
\label{tab:tydiqa_layerwise_overlap}
\end{table}

\begin{table}[t]
\centering
\small
\setlength{\tabcolsep}{4pt}
\begin{tabular}{lcccc}
\toprule
\rowcolor{gray!20}
\textbf{Language} & \textbf{Shal.} & \textbf{Mid.} & \textbf{Deep} & \textbf{Avg.} \\
\midrule
Telugu (TE)     & 0.14 & 0.12 & 0.06 & 0.11 \\
\rowcolor{gray!10}
Bengali (BN)    & 0.24 & 0.16 & 0.00 & 0.13 \\
Arabic (AR)     & 0.19 & 0.23 & 0.00 & 0.14 \\
\rowcolor{gray!10}
Swahili (SW)    & 0.19 & 0.19 & 0.11 & 0.16 \\
Finnish (FI)    & 0.24 & 0.23 & 0.06 & 0.17 \\
\rowcolor{gray!10}
Korean (KO)     & 0.24 & 0.19 & 0.11 & 0.18 \\
Indonesian (ID) & 0.29 & 0.26 & 0.06 & 0.20 \\
\rowcolor{gray!10}
Russian (RU)    & 0.33 & 0.39 & 0.00 & 0.24 \\
\midrule
\textbf{Mean}   & \textbf{0.23} & \textbf{0.22} & \textbf{0.05} & \textbf{0.17} \\
\bottomrule
\end{tabular}
\caption{Layer-wise expert overlap with English on TyDiQA (Phi-3.5-MoE-Instruct).}
\label{tab:phi_tydiqa_layerwise_overlap}
\end{table}

\begin{table}[t]
\centering
\small
\setlength{\tabcolsep}{4pt}
\begin{tabular}{lcccc}
\toprule
\rowcolor{gray!20}
\textbf{Language} & \textbf{Shal.} & \textbf{Mid.} & \textbf{Deep} & \textbf{Avg.} \\
\midrule
Bengali (BN)  & 0.12 & 0.22 & 0.05 & 0.13 \\
\rowcolor{gray!10}
Swahili (SW)  & 0.04 & 0.19 & 0.04 & 0.09 \\
Thai (TH)     & 0.25 & 0.33 & 0.12 & 0.23 \\
\rowcolor{gray!10}
Chinese (ZH)  & 0.36 & 0.43 & 0.11 & 0.30 \\
Japanese (JA) & 0.37 & 0.37 & 0.10 & 0.28 \\
\rowcolor{gray!10}
Russian (RU)  & 0.41 & 0.43 & 0.17 & 0.34 \\
German (DE)   & 0.39 & 0.46 & 0.22 & 0.36 \\
\rowcolor{gray!10}
French (FR)   & 0.46 & 0.49 & 0.24 & 0.40 \\
Spanish (ES)  & 0.52 & 0.51 & 0.25 & 0.43 \\
\midrule
\textbf{Mean} & \textbf{0.32} & \textbf{0.38} & \textbf{0.14} & \textbf{0.28} \\
\bottomrule
\end{tabular}
\caption{Layer-wise expert overlap with English on MGSM (Qwen3-30B-A3B).}
\label{tab:mgsm_layerwise_overlap}
\end{table}

\begin{table}[t]
\centering
\small
\setlength{\tabcolsep}{4pt}
\begin{tabular}{lcccc}
\toprule
\rowcolor{gray!20}
\textbf{Language} & \textbf{Shal.} & \textbf{Mid.} & \textbf{Deep} & \textbf{Avg.} \\
\midrule
Swahili (SW)  & 0.10 & 0.04 & 0.00 & 0.04 \\
\rowcolor{gray!10}
Thai (TH)     & 0.05 & 0.12 & 0.00 & 0.06 \\
Bengali (BN)  & 0.10 & 0.09 & 0.06 & 0.08 \\
\rowcolor{gray!10}
Japanese (JA) & 0.19 & 0.30 & 0.22 & 0.24 \\
Russian (RU)  & 0.29 & 0.40 & 0.11 & 0.27 \\
\rowcolor{gray!10}
German (DE)   & 0.24 & 0.49 & 0.11 & 0.28 \\
French (FR)   & 0.33 & 0.49 & 0.06 & 0.29 \\
\rowcolor{gray!10}
Chinese (ZH)  & 0.05 & 0.51 & 0.39 & 0.32 \\
Spanish (ES)  & 0.38 & 0.46 & 0.11 & 0.32 \\
\midrule
\textbf{Mean} & \textbf{0.19} & \textbf{0.32} & \textbf{0.12} & \textbf{0.21} \\
\bottomrule
\end{tabular}
\caption{Layer-wise expert overlap with English on MGSM (Phi-3.5-MoE-Instruct).}
\label{tab:phi_mgsm_layerwise_overlap}
\end{table}

\paragraph{Qwen3-30B-A3B.}
Tables~\ref{tab:tydiqa_layerwise_overlap} and~\ref{tab:mgsm_layerwise_overlap} report per-language Jaccard similarity with English at the shallow, middle, and deep layers for \texttt{Qwen3-30B-A3B} on TyDiQA and MGSM, respectively.
Besides, we also provide the global-level TyDiQA overlap heatmap in Figure~\ref{fig:qwen_tydiqa_global_overlap}, which supplements Figure~\ref{fig:global_routing} in the main text.

\paragraph{Phi-3.5-MoE-Instruct.}
Tables~\ref{tab:phi_tydiqa_layerwise_overlap} and~\ref{tab:phi_mgsm_layerwise_overlap} report the layer-wise expert overlap statistics for Phi-3.5-MoE-Instruct on TyDiQA and MGSM, respectively.
Figures~\ref{fig:phi_mgsm_global_overlap} and~\ref{fig:phi_tydiqa_global_overlap} provide the corresponding global-level routing overlap analysis, serving as the Phi-side counterparts to Figure~\ref{fig:global_routing}~(a) in the main text.

Despite the substantial architectural differences from Qwen3-30B-A3B (32 vs.\ 48 layers, top-2 vs.\ top-8 routing, 16 vs.\ 128 experts per layer), both the \textit{routing isolation} and \textit{layerwise divergence} phenomena observed in the main text hold consistently: low-resource languages activate a narrow, repetitive set of experts with little overlap at the global level, while the layer-wise similarity follows the same three-phase hierarchy---shallow layers showing moderate overlap, middle layers peaking, and deep layers dropping to near zero.
These results confirm that the two core phenomena identified in our analysis are robust across model families and evaluation tasks, rather than artifacts of a specific architecture or benchmark.

\begin{figure}[htbp]
\centering
\includegraphics[width=\linewidth]{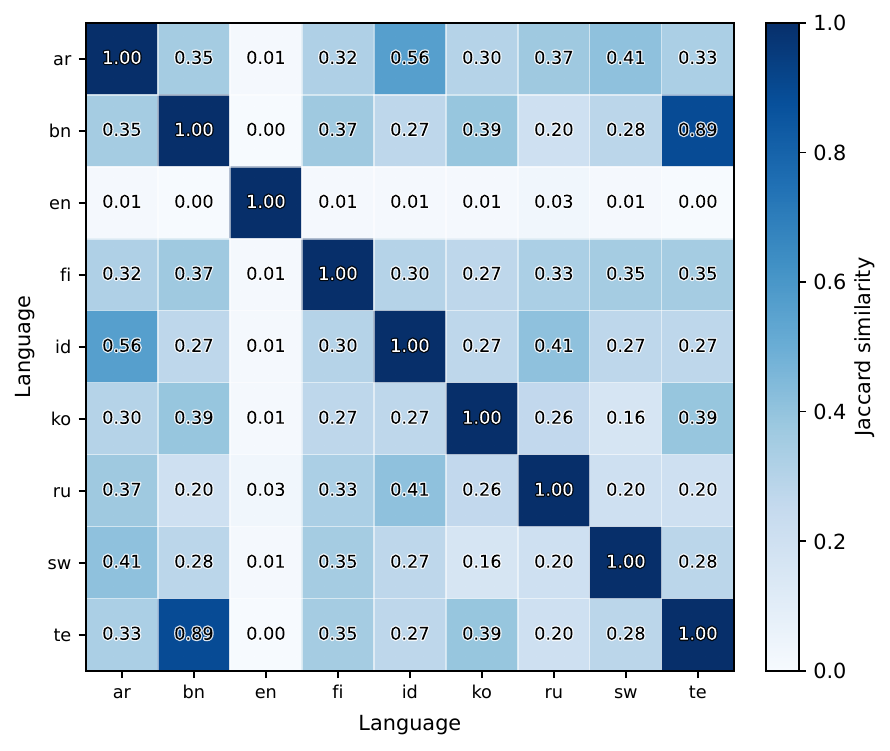}
\caption{Global-level expert activation overlap of Qwen3-30B-A3B across languages in TyDiQA.}
\label{fig:qwen_tydiqa_global_overlap}
\end{figure}

\begin{figure}[htbp]
\centering
\includegraphics[width=\linewidth]{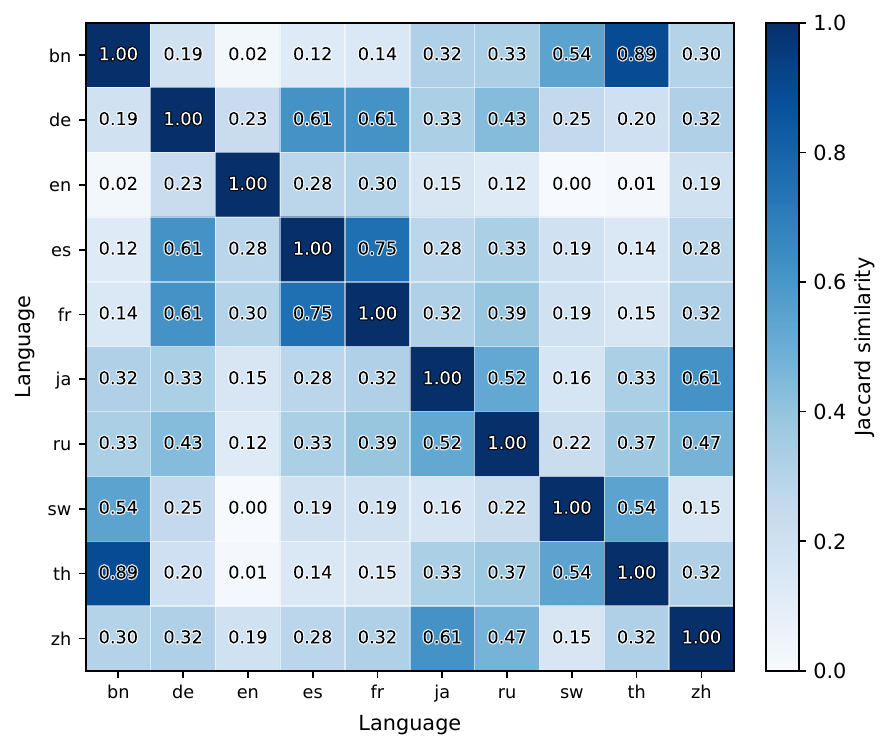}
\caption{Global-level expert activation overlap of Phi-3.5-MoE across languages in MGSM.}
\label{fig:phi_mgsm_global_overlap}
\end{figure}

\begin{figure}[htbp]
\centering
\includegraphics[width=\linewidth]{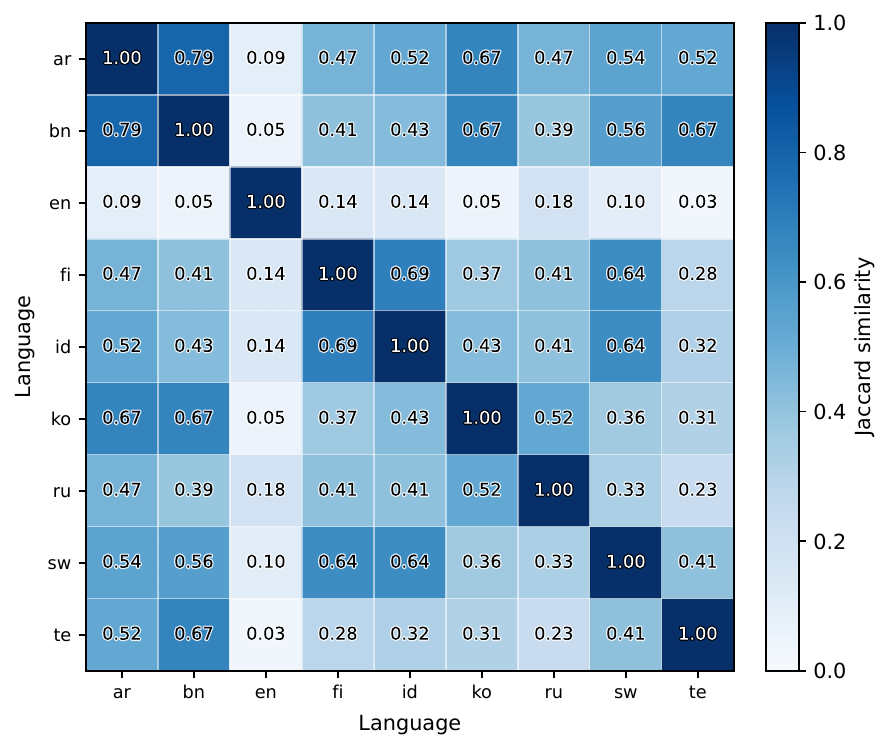}
\caption{Global-level expert activation overlap of Phi-3.5-MoE across languages in TyDiQA.}
\label{fig:phi_tydiqa_global_overlap}
\end{figure}

\section{\rise Selected Expert Visualization}
\label{app:ftse_expert_distribution}

\begin{figure}[h]
\centering
\includegraphics[width=\linewidth]{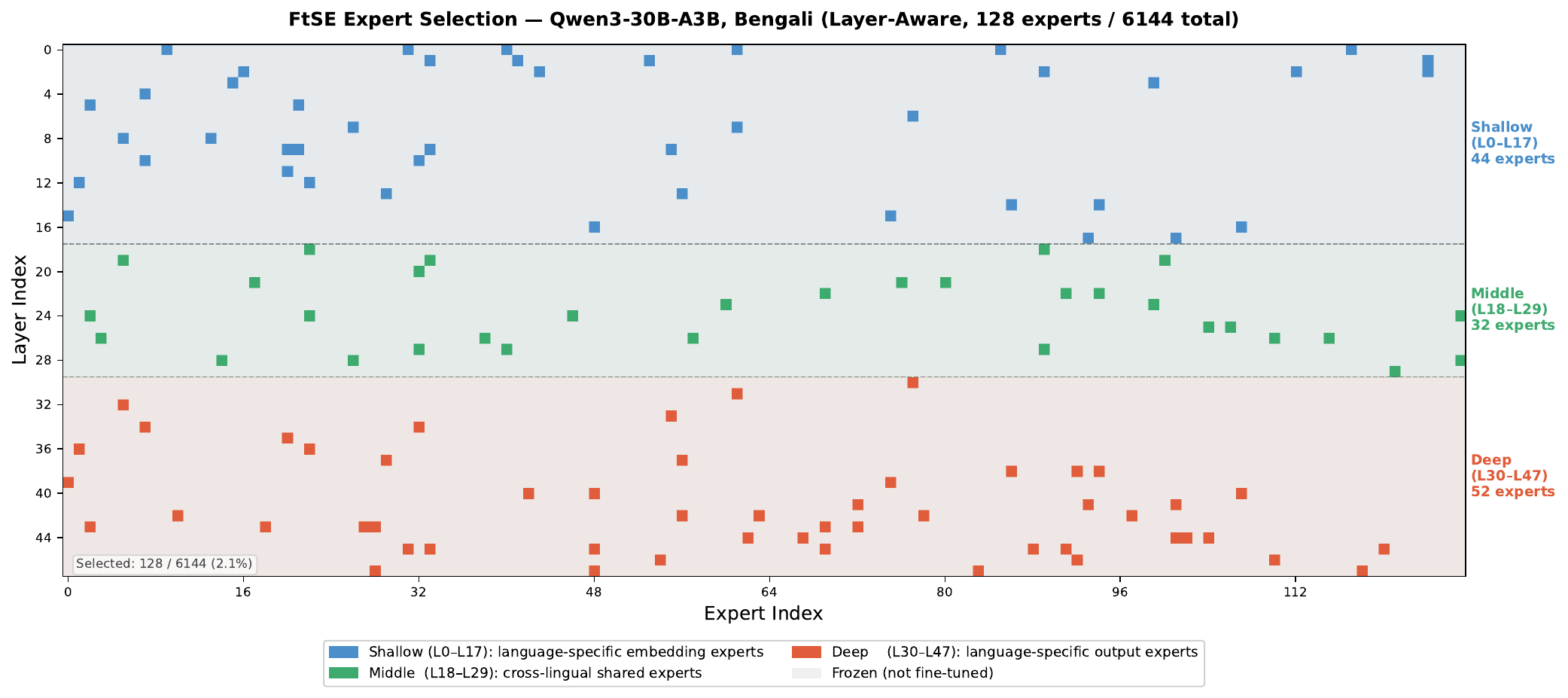}
\caption{Layer-wise expert selection heatmap for Bengali training on \textbf{Qwen3-30B-A3B}. Each column corresponds to a transformer layer and each row to an expert index. Highlighted cells denote experts selected by \rise for gradient updates. The selection density is highest in the middle and deep layers, consistent with the routing isolation observed for Bengali in Table~\ref{tab:mgsm_layerwise_overlap}.}
\label{fig:ftse_expert_heatmap_bengali}
\end{figure}

\begin{figure}[h]
\centering
\includegraphics[width=\linewidth]{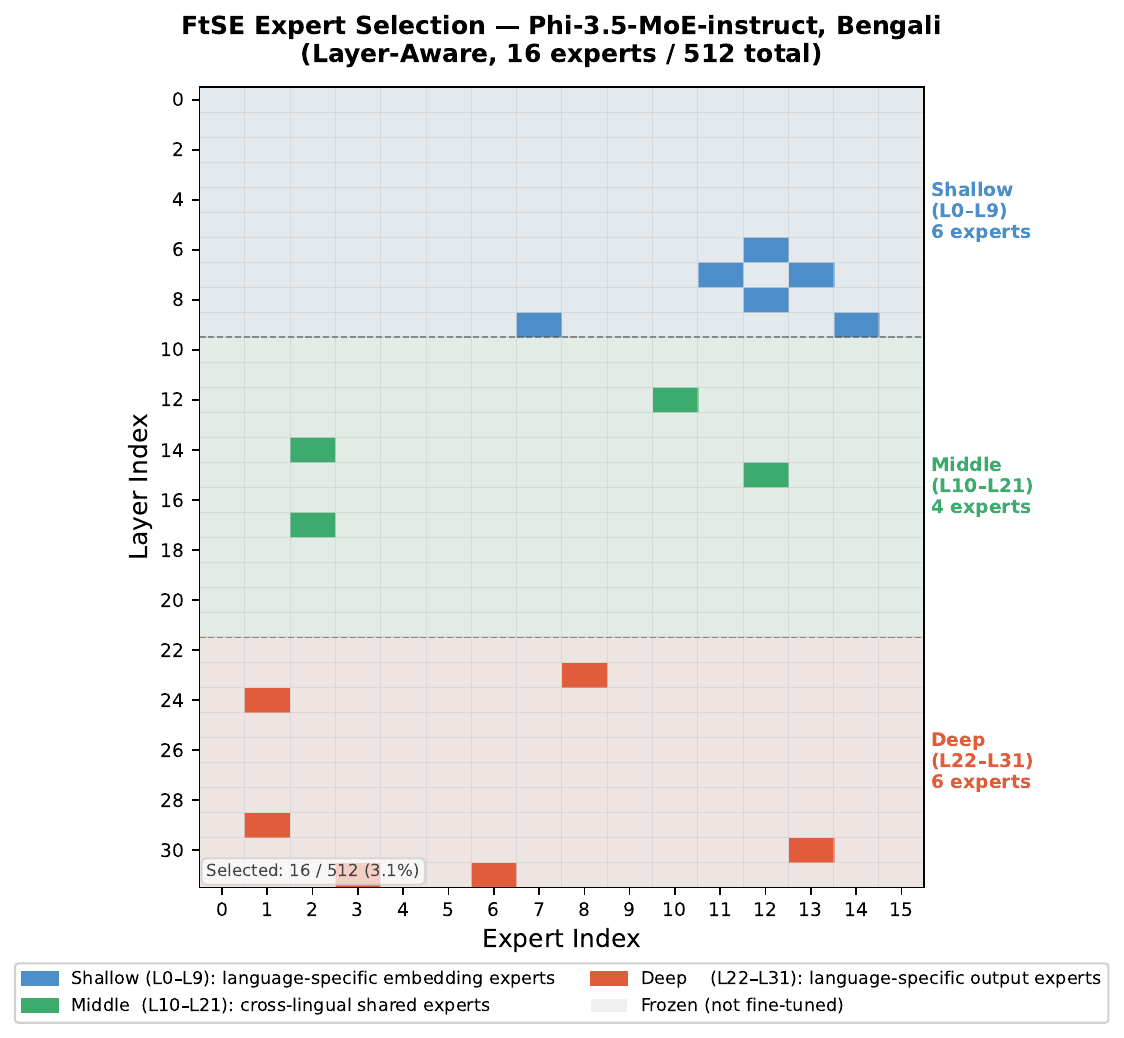}
\caption{Layer-wise expert selection heatmap for Bengali training on \textbf{Phi-3.5-MoE-Instruct}. Despite the architectural differences (32 layers, 16 experts, top-2 routing), \rise again targets experts concentrated in the middle-to-deep range, mirroring the isolation pattern reported in Table~\ref{tab:phi_mgsm_layerwise_overlap}. The sparser overall selection reflects the smaller expert count per layer relative to Qwen3-30B-A3B.}
\label{fig:ftse_expert_heatmap_phi_bengali}
\end{figure}

Figures~\ref{fig:ftse_expert_heatmap_bengali} and~\ref{fig:ftse_expert_heatmap_phi_bengali} visualize the layer-wise distribution of training experts selected by \rise for Bengali on Qwen3-30B-A3B and Phi-3.5-MoE-Instruct, respectively. Each cell indicates whether a given expert in a given layer was selected for training. The heatmaps reveal that \rise concentrates its budget in the middle and deep layers---precisely the layers where Bengali routing diverges most from English (cf.\ Tables~\ref{tab:tydiqa_layerwise_overlap}--\ref{tab:phi_mgsm_layerwise_overlap}). This alignment confirms that the composite score in \rise successfully identifies language-specific experts rather than selecting high-activation experts indiscriminately.

\section{Layer Partition and Expert Selection Settings}
\label{app:layer_settings}

\begin{table*}[t]
\centering
\begin{tabular}{lcc}
\toprule
\textbf{Setting} & \textbf{Qwen3-30B-A3B} & \textbf{Phi-3.5-MoE} \\
\midrule
Total Layers & 48 & 32 \\
Experts per Layer & 128 & 16 \\
Total Expert Pool & 6,144 & 512 \\
\midrule
Shallow Layers & 0--17 (37.5\%) & 0--9 (31.25\%) \\
Middle Layers  & 18--29 (25\%)  & 10--21 (37.5\%) \\
Deep Layers    & 30--47 (37.5\%) & 22--31 (31.25\%) \\
\midrule
Selected Experts $K$ & 128 (2.1\%) & 16 (3.1\%) \\
\quad Shallow budget & 35\% ($\approx$44) & 37.5\% ($\approx$6) \\
\quad Middle budget  & 25\% ($\approx$32) & 25\% ($\approx$4) \\
\quad Deep budget    & 40\% ($\approx$52) & 37.5\% ($\approx$6) \\
\midrule
Shallow criterion & Language specificity & Language specificity \\
Middle criterion  & Cross-lingual overlap & Cross-lingual overlap \\
Deep criterion    & Language specificity & Language specificity \\
\bottomrule
\end{tabular}
\caption{Layer-aware expert selection settings for Qwen3-30B-A3B and Phi-3.5-MoE-Instruct.}
\label{tab:layer_aware_settings}
\end{table*}

Table~\ref{tab:layer_aware_settings} summarizes the layer partition boundaries and expert selection budget used in \rise for both backbone models.
The three functional groups---shallow (language-specific encoding), middle (language-agnostic semantic processing), and deep (language-specific generation)---follow the layer boundaries derived from our routing divergence analysis in Section~\ref{sec:analysis}.
The expert budget $K$ is allocated asymmetrically across groups, concentrating more budget on shallow and deep layers than on the middle layers.

\begin{table}[htbp]
\centering
\small
\setlength{\tabcolsep}{2pt}
\begin{tabular}{@{}cccccc@{}}
\toprule
\textbf{$K$} & \textbf{Weights} & \textbf{Act.} & \textbf{Grad.} & \textbf{Opt.} & \textbf{Total} \\
\midrule
8   & 56.3 GB & 7.9 GB & 72 MB  & 432 MB & \textbf{64.7 GB} \\
16  & 56.3 GB & 7.9 GB & 144 MB & 864 MB & \textbf{65.1 GB} \\
32  & 56.3 GB & 7.9 GB & 288 MB & 1.7 GB & \textbf{66.1 GB} \\
64  & 56.3 GB & 7.9 GB & 576 MB & 3.4 GB & \textbf{68.1 GB} \\
128 & 56.3 GB & 7.9 GB & 1.1 GB & 6.8 GB & \textbf{72.0 GB} \\
\bottomrule
\end{tabular}
\caption{GPU memory usage during training for different numbers of selected experts $K$ on Qwen3-30B-A3B. Note that these figures represent the \emph{minimum} memory footprint without any training data loaded; actual memory consumption may increase significantly with larger batch sizes.}
\label{tab:gpu_memory_usage}
\end{table}

\section{Relationship Between \rise and LoRA}
\label{app:why_not_lora}

We do not treat LoRA as a competing baseline for \rise because the two operate at orthogonal levels of model adaptation and address fundamentally different questions; Section~\ref{sec:rise_lora} instead reports what happens when they are composed.

\rise operates at the \textbf{routing level}: it identifies \emph{which experts} to update, based on a mechanistic analysis of multilingual routing behavior in MoE models. LoRA operates at the \textbf{parameter-update level}: it specifies \emph{how} to update the weight matrices efficiently via low-rank adapters. These are orthogonal design dimensions---\rise concerns the \emph{selection scope} of adaptation, while LoRA concerns the \emph{parameterization} of the update.

\begin{figure}[!t]
    \centering
    \includegraphics[width=\linewidth]{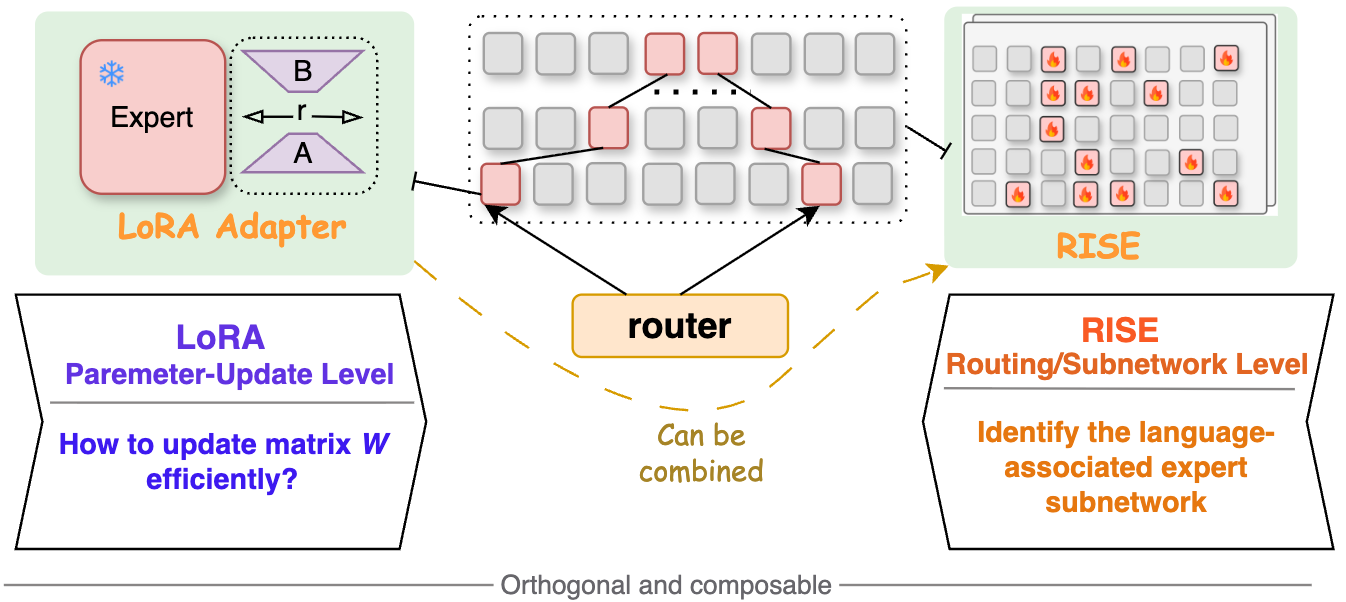}
    \caption{\rise and LoRA are orthogonal and composable: \rise operates at the \textbf{routing level} (which experts to adapt), while LoRA operates at the \textbf{parameter-update level} (how to update them).}
    \label{fig:rise_vs_lora}
\end{figure}

The main contribution of this work is the discovery that multilingual MoE models exhibit structured routing isolation, from which language-specific expert subnetworks can be identified in a principled manner. We intentionally train the selected experts with standard full-rank updates to isolate the contribution of expert selection itself, avoiding confounds introduced by rank constraints or adapter initialization.

Importantly, \rise and LoRA are fully compatible rather than mutually exclusive. One may readily apply LoRA, QLoRA, or other PEFT techniques \emph{within} the \rise-selected subnetwork---freezing all non-selected experts and attaching low-rank adapters only to the selected ones. Under this view, \rise answers \emph{which part} of the MoE model to adapt, while LoRA answers \emph{how} to update that part efficiently. Figure~\ref{fig:rise_vs_lora} illustrates this orthogonal and composable relationship.

For this reason, treating LoRA as a direct competing baseline would conflate two orthogonal design dimensions and blur the central contribution of this paper. We note that prior work has already explored LoRA-style adaptation within MoE training~\citep{moelora}, further supporting the view that LoRA is a natural \emph{complement} to \rise---applicable within the selected subnetwork for additional parameter efficiency---rather than an alternative to it.

\enlargethispage{2\baselineskip}

\input{contents/big_case}

%% file: contents/big_case.tex
\tcbset{
  casebox/.style={
    colback=gray!4, colframe=gray!50,
    fonttitle=\bfseries\small,
    boxrule=0.7pt, arc=3pt,
    top=3pt, bottom=3pt, left=4pt, right=4pt,
    before skip=2pt, after skip=2pt
  },
  vbox/.style={
    colback=red!5, colframe=red!45,
    title={\footnotesize\bfseries\color{red!65!black}Vanilla\ \xmark},
    boxrule=0.5pt, arc=2pt,
    top=2pt, bottom=2pt, left=3pt, right=3pt,
    before skip=2pt, after skip=1pt
  },
  fbox/.style={
    colback=green!5, colframe=green!50!black!55,
    title={\footnotesize\bfseries\color{green!50!black}\rise (Ours)\ \cmark},
    boxrule=0.5pt, arc=2pt,
    top=2pt, bottom=2pt, left=3pt, right=3pt,
    before skip=1pt, after skip=1pt
  }
}

\section{Case Study}
\label{case_study}
\makeatletter
\setlength{\@fptop}{0pt}
\setlength{\@dblfptop}{0pt}
\makeatother
\begin{figure*}[p]
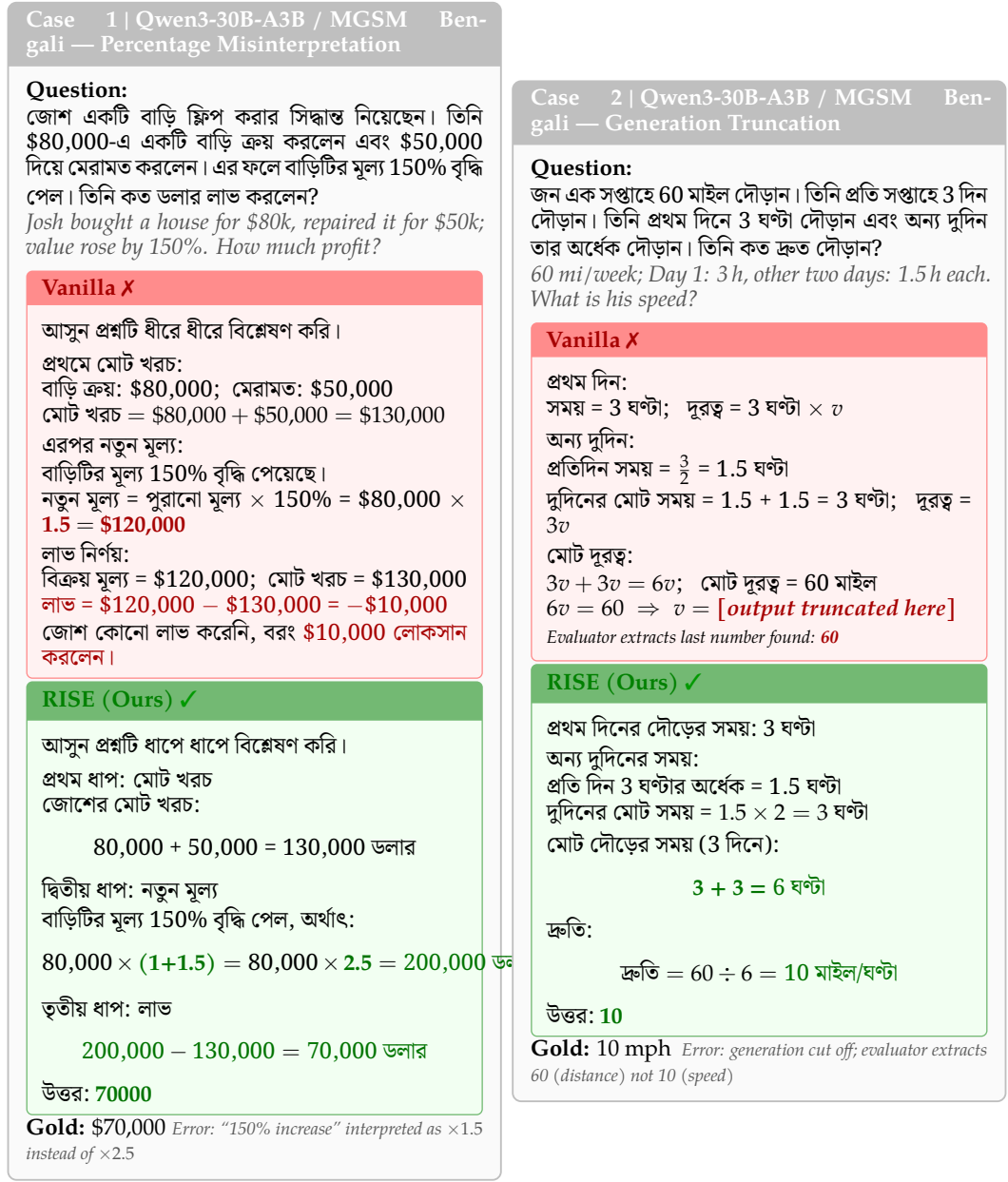

\centering\small

\begin{minipage}[c]{0.495\linewidth}

\begin{tcolorbox}[casebox,
  title={Case 1\;|\;Qwen3-30B-A3B\;/\;MGSM Bengali\;---\;Percentage Misinterpretation}]

\textbf{Question:}\\[1pt]
\beng{জোশ একটি বাড়ি ফ্লিপ করার সিদ্ধান্ত নিয়েছেন। তিনি \$80,000-এ একটি বাড়ি ক্রয় করলেন
এবং \$50,000 দিয়ে মেরামত করলেন। এর ফলে বাড়িটির মূল্য 150\% বৃদ্ধি পেল।
তিনি কত ডলার লাভ করলেন?}\\
{\color{gray!65!black}\textit{Josh bought a house for \$80k, repaired it for \$50k;
value rose by 150\%. How much profit?}}

\vspace{2pt}
\begin{tcolorbox}[vbox]
\beng{আসুন প্রশ্নটি ধীরে ধীরে বিশ্লেষণ করি।}\\[2pt]
{\bengalifont প্রথমে মোট খরচ:}\\
\beng{বাড়ি ক্রয়: \$80,000;~~মেরামত: \$50,000}\\
\beng{মোট খরচ} $= \$80{,}000 + \$50{,}000 = \$130{,}000$\\[2pt]
{\bengalifont এরপর নতুন মূল্য:}\\
\beng{বাড়িটির মূল্য 150\% বৃদ্ধি পেয়েছে।}\\
\beng{নতুন মূল্য = পুরানো মূল্য $\times$ 150\% =
  \$80,000 $\times$} \herr{1.5} $=$ \herr{\$120{,}000}\\[2pt]
{\bengalifont লাভ নির্ণয়:}\\
\beng{বিক্রয় মূল্য = \$120,000;~~মোট খরচ = \$130,000}\\
{\color{red!65!black}\beng{লাভ = \$120,000 $-$ \$130,000 = $-$\$10,000}}\\[1pt]
\beng{জোশ কোনো লাভ করেনি, বরং}~{\color{red!65!black}\beng{\$10,000 লোকসান করলেন।}}
\end{tcolorbox}

\begin{tcolorbox}[fbox]
\beng{আসুন প্রশ্নটি ধাপে ধাপে বিশ্লেষণ করি।}\\[2pt]
{\bengalifont প্রথম ধাপ: মোট খরচ}\\
\beng{জোশের মোট খরচ:}
\[\text{\beng{80{,}000 + 50{,}000 = 130{,}000~ডলার}}\]
{\bengalifont দ্বিতীয় ধাপ: নতুন মূল্য}\\
\beng{বাড়িটির মূল্য 150\% বৃদ্ধি পেল, অর্থাৎ:}
\[\text{\beng{80{,}000}} \times \hcorr{(1+1.5)} = \text{\beng{80{,}000}} \times
  \hcorr{2.5} = {\color{green!45!black}\text{\beng{200{,}000~ডলার}}}\]
{\bengalifont তৃতীয় ধাপ: লাভ}
\[{\color{green!45!black}\text{\beng{200{,}000}}} - {\color{green!45!black}\text{\beng{130{,}000}}} =
  {\color{green!45!black}\text{\beng{70{,}000~ডলার}}}\]
\beng{উত্তর:}~\hcorr{70000}
\end{tcolorbox}

\noindent\textbf{Gold:} \$70,000\hfill
{\scriptsize\color{gray!70!black}\textit{Error: ``150\% increase'' interpreted as $\times1.5$ instead of $\times2.5$}}
\end{tcolorbox}

\end{minipage}
\hfill
\begin{minipage}[c]{0.495\linewidth}

\begin{tcolorbox}[casebox,
  title={Case 2\;|\;Qwen3-30B-A3B\;/\;MGSM Bengali\;---\;Generation Truncation}]

\textbf{Question:}\\[1pt]
\beng{জন এক সপ্তাহে 60 মাইল দৌড়ান। তিনি প্রতি সপ্তাহে 3 দিন দৌড়ান।
তিনি প্রথম দিনে 3 ঘণ্টা দৌড়ান এবং অন্য দুদিন তার অর্ধেক দৌড়ান।
তিনি কত দ্রুত দৌড়ান?}\\
{\color{gray!65!black}\textit{60 mi/week; Day~1: 3\,h, other two days: 1.5\,h each.
What is his speed?}}

\vspace{2pt}
\begin{tcolorbox}[vbox]
{\bengalifont প্রথম দিন:}\\
\beng{সময় = 3 ঘণ্টা;~~~দূরত্ব = 3 ঘণ্টা $\times$ $v$}\\[2pt]
{\bengalifont অন্য দুদিন:}\\
\beng{প্রতিদিন সময় = $\frac{3}{2}$ = 1.5 ঘণ্টা}\\
\beng{দুদিনের মোট সময় = 1.5 + 1.5 = 3 ঘণ্টা;~~~দূরত্ব = $3v$}\\[2pt]
{\bengalifont মোট দূরত্ব:}\\
\beng{$3v + 3v = 6v$;~~~মোট দূরত্ব = 60 মাইল}\\
$6v = 60 \;\Rightarrow\; v = \;$\herr{\textit{[output truncated here]}}\\[1pt]
{\scriptsize\textit{Evaluator extracts last number found: \herr{60}}}
\end{tcolorbox}

\begin{tcolorbox}[fbox]
{\bengalifont প্রথম দিনের দৌড়ের সময়:}~\beng{3 ঘণ্টা}\\[2pt]
{\bengalifont অন্য দুদিনের সময়:}\\
\beng{প্রতি দিন 3 ঘণ্টার অর্ধেক = 1.5 ঘণ্টা}\\
\beng{দুদিনের মোট সময় = $1.5 \times 2 = 3$ ঘণ্টা}\\[2pt]
{\bengalifont মোট দৌড়ের সময় (3 দিনে):}
\[{\color{green!45!black}3 + 3 = \text{\beng{6~ঘণ্টা}}}\]
{\bengalifont দ্রুতি:}
\[\text{\beng{দ্রুতি}} = 60 \div 6 = {\color{green!45!black}\text{\beng{10~মাইল/ঘণ্টা}}}\]
\beng{উত্তর:}~\hcorr{10}
\end{tcolorbox}

\noindent\textbf{Gold:} 10 mph\hfill
{\scriptsize\color{gray!70!black}\textit{Error: generation cut off; evaluator extracts 60 (distance) not 10 (speed)}}
\end{tcolorbox}

\end{minipage}

\caption{Qualitative case study (\textbf{Cases 1--2}) comparing \textbf{Vanilla} (base model)
and \textbf{\rise} (ours) on Bengali MGSM\@.
\textcolor{red!65!black}{\textbf{Red bold}} marks incorrect steps;
\textcolor{green!45!black}{\textbf{green bold}} marks correct reasoning.
Both cases show mathematical reasoning errors in Qwen3-30B-A3B:
Case~1 misinterprets ``150\% increase'' as $\times1.5$ instead of $\times2.5$;
Case~2 suffers premature generation truncation, causing the evaluator to extract
the distance (60) rather than the speed (10).
See Figure~\ref{fig:case_study_cd} for Cases 3--4.}
\label{fig:case_study}
\end{figure*}

\begin{figure*}[p]
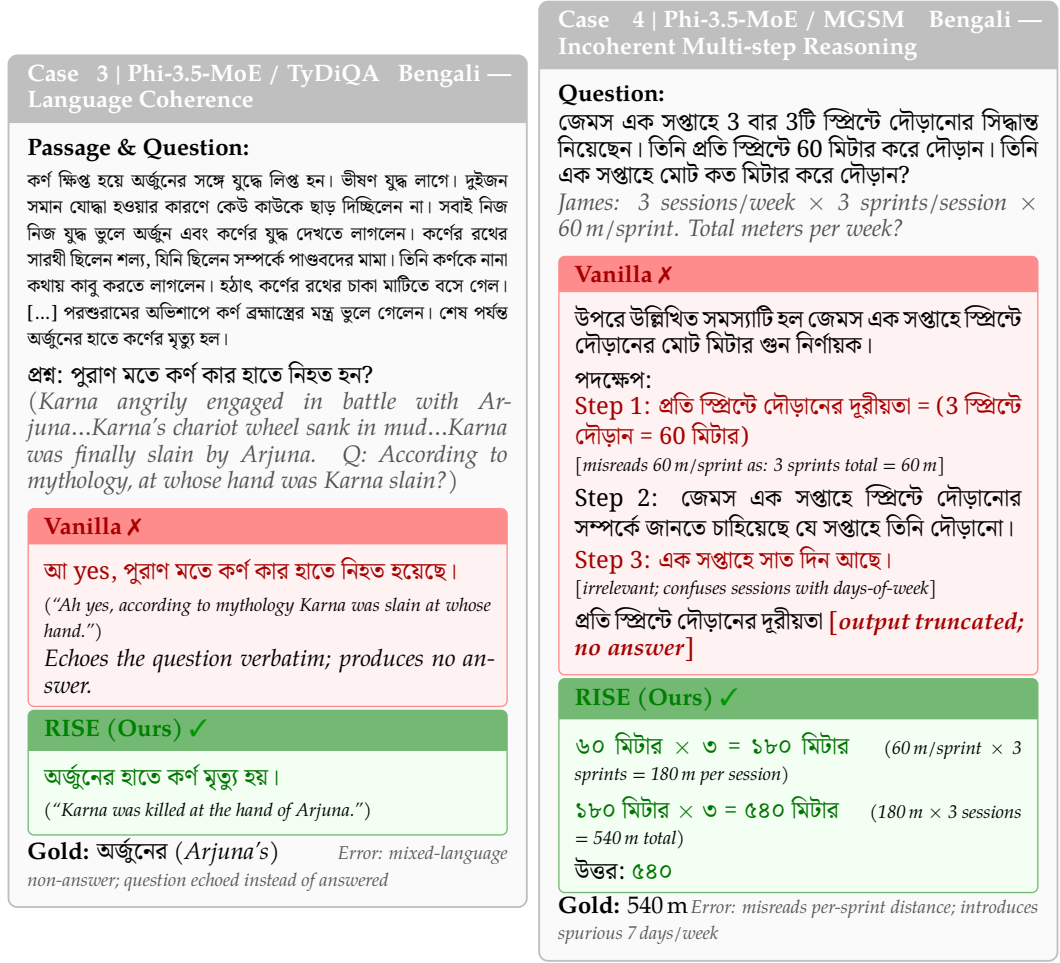

\centering\small

\begin{minipage}[c]{0.495\linewidth}

\begin{tcolorbox}[casebox,
  title={Case 3\;|\;Phi-3.5-MoE\;/\;TyDiQA Bengali\;---\;Language Coherence}]

\textbf{Passage \& Question:}\\[2pt]
{\scriptsize\beng{
কর্ণ ক্ষিপ্ত হয়ে অর্জুনের সঙ্গে যুদ্ধে লিপ্ত হন। ভীষণ যুদ্ধ লাগে। দুইজন সমান যোদ্ধা
হওয়ার কারণে কেউ কাউকে ছাড় দিচ্ছিলেন না। সবাই নিজ নিজ যুদ্ধ ভুলে অর্জুন এবং কর্ণের
যুদ্ধ দেখতে লাগলেন। কর্ণের রথের সারথী ছিলেন শল্য, যিনি ছিলেন সম্পর্কে পাণ্ডবদের মামা।
তিনি কর্ণকে নানা কথায় কাবু করতে লাগলেন। হঠাৎ কর্ণের রথের চাকা মাটিতে বসে গেল।
[\ldots] পরশুরামের অভিশাপে কর্ণ ব্রহ্মাস্ত্রের মন্ত্র ভুলে গেলেন।
শেষ পর্যন্ত অর্জুনের হাতে কর্ণের মৃত্যু হল।}}\\[3pt]
{\small{\bengalifont প্রশ্ন:} \beng{পুরাণ মতে কর্ণ কার হাতে নিহত হন?}}\\
{\color{gray!65!black}\textit{(Karna angrily engaged in battle with Arjuna\ldots{}Karna's
chariot wheel sank in mud\ldots{}Karna was finally slain by Arjuna.
Q: According to mythology, at whose hand was Karna slain?)}}

\smallskip
\begin{tcolorbox}[vbox]
{\color{red!65!black}\beng{আ yes, পুরাণ মতে কর্ণ কার হাতে নিহত হয়েছে।}}\\[2pt]
{\scriptsize\textit{(``Ah yes, according to mythology Karna was slain at whose hand.'')}}\\[1pt]
\textit{Echoes the question verbatim; produces no answer.}
\end{tcolorbox}

\begin{tcolorbox}[fbox]
{\color{green!45!black}\beng{অর্জুনের হাতে কর্ণ মৃত্যু হয়।}}\\[2pt]
{\scriptsize\textit{(``Karna was killed at the hand of Arjuna.'')}}
\end{tcolorbox}

\noindent\textbf{Gold:} \beng{অর্জুনের} \textit{(Arjuna's)}\hfill
{\scriptsize\color{gray!70!black}\textit{Error: mixed-language non-answer; question echoed instead of answered}}
\end{tcolorbox}

\end{minipage}
\hfill
\begin{minipage}[c]{0.495\linewidth}

\begin{tcolorbox}[casebox,
  title={Case 4\;|\;Phi-3.5-MoE\;/\;MGSM Bengali\;---\;Incoherent Multi-step Reasoning}]

\textbf{Question:}\\[1pt]
\beng{জেমস এক সপ্তাহে 3 বার 3টি স্প্রিন্টে দৌড়ানোর সিদ্ধান্ত নিয়েছেন।
তিনি প্রতি স্প্রিন্টে 60 মিটার করে দৌড়ান।
তিনি এক সপ্তাহে মোট কত মিটার করে দৌড়ান?}\\
{\color{gray!65!black}\textit{James: 3 sessions/week $\times$ 3 sprints/session $\times$
60\,m/sprint. Total meters per week?}}

\smallskip
\begin{tcolorbox}[vbox]
\beng{উপরে উল্লিখিত সমস্যাটি হল জেমস এক সপ্তাহে স্প্রিন্টে
দৌড়ানের মোট মিটার গুন নির্ণায়ক।}\\[3pt]
{\bengalifont পদক্ষেপ:}\\
{\color{red!65!black}\beng{Step~1: প্রতি স্প্রিন্টে দৌড়ানের দূরীয়তা =
  (3 স্প্রিন্টে দৌড়ান = 60 মিটার)}}\\
{\scriptsize\textit{[misreads 60\,m/sprint as: 3 sprints total = 60\,m]}}\\[2pt]
\beng{Step~2: জেমস এক সপ্তাহে স্প্রিন্টে দৌড়ানোর সম্পর্কে
  জানতে চাহিয়েছে যে সপ্তাহে তিনি দৌড়ানো।}\\[2pt]
{\color{red!65!black}\beng{Step~3: এক সপ্তাহে সাত দিন আছে।}}\\
{\scriptsize\textit{[irrelevant; confuses sessions with days-of-week]}}\\[2pt]
\beng{প্রতি স্প্রিন্টে দৌড়ানের দূরীয়তা}~\herr{\textit{[output truncated; no answer]}}
\end{tcolorbox}

\begin{tcolorbox}[fbox]
{\color{green!45!black}\beng{৬০ মিটার $\times$ ৩ = ১৮০ মিটার}}
\quad{\scriptsize\textit{(60\,m/sprint $\times$ 3 sprints = 180\,m per session)}}\\[3pt]
{\color{green!45!black}\beng{১৮০ মিটার $\times$ ৩ = ৫৪০ মিটার}}
\quad{\scriptsize\textit{(180\,m $\times$ 3 sessions = 540\,m total)}}\\[3pt]
\beng{উত্তর:~}{\color{green!45!black}\beng{৫৪০}}
\end{tcolorbox}

\noindent\textbf{Gold:} 540\,m\hfill
{\scriptsize\color{gray!70!black}\textit{Error: misreads per-sprint distance; introduces spurious 7\,days/week}}
\end{tcolorbox}

\end{minipage}

\caption{Qualitative case study (\textbf{Cases 3--4}) comparing \textbf{Vanilla} (base model)
and \textbf{\rise} (ours) on Bengali TyDiQA and MGSM\@.
\textcolor{red!65!black}{\textbf{Red bold}} marks incorrect steps;
\textcolor{green!45!black}{\textbf{green bold}} marks correct reasoning.
Both cases show incoherent or mixed-language outputs in Phi-3.5-MoE:
Case~3 produces a question-echoing non-answer with mixed-language output;
Case~4 misreads per-sprint distance and introduces a spurious 7-day/week step,
with output truncated before any answer is given.
\rise{}'s Bengali-specialized expert activation corrects both failures.}
\label{fig:case_study_cd}
\end{figure*}